\documentclass{article} 
\usepackage{iclr2019_conference,times}

\usepackage{amsmath,amsfonts,bm}

\def\eqref#1{equation~\ref{#1}}

\def\1{\bm{1}}

\DeclareMathAlphabet{\mathsfit}{\encodingdefault}{\sfdefault}{m}{sl}
\SetMathAlphabet{\mathsfit}{bold}{\encodingdefault}{\sfdefault}{bx}{n}

\usepackage[utf8]{inputenc} 
\usepackage[T1]{fontenc}    
\usepackage{hyperref}       
\usepackage{url}            
\usepackage{booktabs}       
\usepackage{amsfonts}       
\usepackage{nicefrac}       
\usepackage{microtype}      

\usepackage{amsmath, amssymb}
\usepackage{amsthm}
\usepackage{algorithm}
\usepackage{algorithmic}
\usepackage{epsfig}
\usepackage{graphicx}
\usepackage{subfigure}
\usepackage{wrapfig}
\usepackage{xspace}
\usepackage{enumitem}

\newtheorem{ppt}{Property}

\newtheorem{prop}{Proposition}

\newtheorem{thm}{Theorem}

\newcommand{\algo}{\textsc{Graph2Seq}\xspace}
\newcommand{\algom}{Graph2Seq\xspace}
\newcommand{\model}{\textsc{local-gather}\xspace}
\newcommand{\gtos}{\textsc{G2S-RNN}\xspace}

\title{\algo: Scalable Learning Dynamics for Graphs}

\author{Shaileshh Bojja Venkatakrishnan, Mohammad Alizadeh \\
Massachusetts Institute of Technology \\
\texttt{\{bjjvnkt,alizadeh\}@csail.mit.edu} \\
\AND
Pramod Viswanath\\
University of Illinois Urbana-Champaign \\
\texttt{pramodv@illinois.edu} \\
}

\iclrfinalcopy 
\begin{document}

\maketitle

\begin{abstract}
Neural networks have been shown to be an effective tool for learning algorithms over graph-structured data. However, {\em graph representation} techniques---that convert graphs to real-valued vectors for use with neural networks---are still in their infancy. Recent works have proposed several approaches (e.g., graph convolutional networks), but these methods have difficulty scaling and generalizing to graphs with different sizes and shapes.
We present \algo, a new technique that represents vertices of graphs as infinite time-series. By not limiting the representation to a fixed dimension, \algo scales naturally to graphs of arbitrary sizes and shapes. 
\algo is also reversible, allowing full recovery of the graph structure from the sequences. By analyzing a formal computational model for graph representation, we show that an unbounded sequence is necessary for scalability. Our experimental results with \algo show strong generalization and new state-of-the-art performance on a variety of graph combinatorial optimization problems.
\end{abstract}

\section{Introduction}


Graph algorithms appear in a wide variety of fields and applications, from the study of gene interactions~\citep{ozgur2008identifying} to social networks~\citep{ugander2011anatomy} to computer systems~\citep{grandl2016g}. Today, most graph algorithms are designed by human experts. However, in many applications, designing graph algorithms with strong performance guarantees is very challenging. These algorithms often involve difficult combinatorial optimization problems for which finding optimal solutions is computationally intractable or current algorithmic understanding is limited (e.g., approximation gaps in CS theory literature).

In recent years, deep learning has achieved impressive results on many tasks, from object recognition to language translation to learning complex heuristics directly from data~\citep{silver2016mastering, krizhevsky2012imagenet}. It is thus natural to ask whether we can apply deep learning to automatically learn complex graph algorithms. 
To apply deep learning to such problems, graph-structured data first needs to be {\em embedded} in a high dimensional Euclidean space. 
 {\em  Graph representation} refers to the problem of embedding graphs or their vertices/edges  in Euclidean spaces. 
Recent works have proposed several graph representation  techniques, notably, a family of representations called graph convolutional neural networks (GCNN)
 that use architectures inspired by CNNs for images~\citep{bruna2013spectral, monti2016geometric, dai2017learning, niepert2016learning, defferrard2016convolutional, hamilton2017representation, bronstein2017geometric, bruna2017community}. 
 Some GCNN representations capture signals on a fixed graph while others support varying sized graphs.
 
In this paper, we consider how to learn graph algorithms in a way that generalizes to large graphs  and graphs of different topologies. We ask: {\em Can a graph neural network trained at a certain scale perform well on orders-of magnitude larger graphs (e.g., $1000 \times$ the size of training graphs) from diverse topologies?} 
We particularly focus on learning algorithms for combinatorial optimization on graphs. 
Prior works have typically considered applications where the training and test graphs are similar in scale, shape, and attributes, and consequently have not addressed this generalization problem.
For example, \citet{dai2017learning} train models on graph of size 50-100 and test on graphs of up to size 1200 from the same family;
\citet{hamilton2017inductive} propose inductive learning methods where they train on graphs with $15k$ edges, and test on graphs with $35k$ edges (less than $3\times$ generalization in size);  \citet{ying2018graph} train and test over the same large Pinterest graph ($3$ billion nodes).

We propose \algo, a {\em scalable} embedding that represents vertices of a graph as a {\em time-series} (\S\ref{sec: rep graphs as ts}). Our key insight is that the fixed-sized vector representation produced by prior GCNN designs limits scalability. Instead, \algo uses the entire time-series of vectors produced by graph convolution layers as the vertex representation. This approach has two benefits: (1) it can capture subgraphs of increasing diameter around each vertex as the timeseries evolves;  (2) it allows us to vary the dimension of the vertex representation based on the input graph; for example, we can use a small number of graph convolutions during training with small graphs and perform more convolutions at test time for larger graphs.  We show both theoretically and empirically that this time-series representation significantly improves the scalability and generalization of the model. Our framework is general and can be applied to various existing GCNN architectures.




We prove that \algo~is information-theoretically lossless, i.e., the graph can be fully recovered from the time-series representations of its vertices (\S\ref{sec: grep proof}). Our proof leverages mathematical connections between \algo and causal inference theory~\citep{granger1980testing, rahimzamani2016network, quinn2015directed}. Further, we show that \algo and many previous GCNN variants are all examples of a certain computational model over graphs that we call \model, providing for a conceptual and algorithmic unification. Using this computational model, we prove that unlike \algo, fixed-length representations fundamentally cannot compute certain functions over graphs.




To apply \algo, we combine graph convolutions with an appropriate RNN that processes the time-series representations (\S\ref{sec: Neural Network Design}). We use this neural network model, \gtos, to tackle three classical combinatorial optimization problems of varying difficulty using reinforcement learning: {\em minimum vertex cover, maximum cut} and {\em maximum independent set} (\S\ref{sec: evals}).  Our experiments show that \algo performs as well or better than the best non-learning heuristic on all three problems and exhibits significantly better scalability and generalization than previous state-of-the-art GCNN~\citep{dai2017learning, hamilton2017inductive, lei2017deriving} or graph kernel based~\citep{shervashidze2011weisfeiler} representations. Highlights of our experimental findings include: 
\begin{enumerate}[noitemsep, topsep=0pt]
\item \gtos~models trained on graphs of size 15--20 scale to graphs of size 3,200 and beyond. To conduct experiments in a reasonable time, we used graphs of size up to 3,200 in most experiments. However, stress tests show similar scalability even at 25,000 vertices.  
\item \gtos~models trained on one graph type (e.g., Erdos-Renyi) generalize to other graph types (e.g., random regular and bipartite graphs). 
\item \gtos exhibits strong scalability and generalization in each of minimum vertex cover, maximum cut and maximum independent set problems. 
\item Training over a carefully chosen adversarially set of graph examples further boosts \gtos's scalability and generalization capabilities. 
\end{enumerate}

\section{Related Work}

{\bf Neural networks on graphs.}
Early works to apply neural-network-based learning to graphs are~\citet{gori2005new, scarselli2009graph}, which consider an information diffusion mechanism.
The notion of convolutional networks for graphs as a generalization of classical convolutional networks for images was introduced by \citet{bruna2013spectral}.
A key contribution of this work is the definition of graph convolution in the {\em spectral domain} using graph Fourier transform theory. Subsequent works have developed local spectral convolution techniques that are easier to compute~\citep{defferrard2016convolutional, kipf2016semi}.
Spectral approaches do not generalize readily to different graphs due to their reliance on the particular Fourier basis on which they were trained.
To address this limitation, recent works have considered {\em spatial convolution} methods~\citep{dai2017learning, monti2016geometric, niepert2016learning, such2017robust, duvenaud2015convolutional, atwood2016diffusion}.
\cite{li2015gated, johnson2016learning} propose a variant that uses gated recurrent units to perform the state updates, which has some similarity to our representation dynamics; however, the sequence length is fixed between training and testing.
\citet{velivckovic2017graph, hamilton2017inductive} use additional aggregation methods such as vertex attention or pooling mechanisms to summarize neighborhood states. 
In Appendix~\ref{sec: background} we show that local spectral GCNNs and spatial GCNNs are mathematically equivalent, providing a unifying view of the variety of GCNN representations in the literature.

Another line of work~\citep{jain2016structural, marcheggiani2017encoding, tai2015improved} combines graph neural networks with RNN modules.
They are not related to our approach, since in these cases the sequence (e.g., time-series of object relationship graphs from a video) is already given as part of the input. 
In contrast our approach generates a sequence as the desired embedding from a single input graph. 
\citet{perozzi2014deepwalk, grover2016node2vec} use random walks to learn vertex representations in an unsupervised or semi-supervised fashion.
However they consider prediction or classification tasks over a fixed graph. 


\smallskip
\noindent
{\bf Combinatorial optimization.}
Using neural networks for combinatorial optimization problems dates back to the work of~\citet{hopfield1985neural} and has received considerable attention in the deep learning community in recent years. \citet{vinyals2015pointer, bello2016neural, kool2018attention} consider the traveling salesman problem using reinforcement learning. These papers consider two-dimensional coordinates for vertices (e.g. cities on a map), without any explicit graph structure.
\citet{graves2016hybrid} propose a more general approach: a differential neural computer that is able to perform tasks like finding the shortest path in a graph.
The work of \citet{dai2017learning} is closest to ours. It applies a spatial GCNN representation in a reinforcement learning framework to solve combinatorial optimization problems such as minimum vertex cover. 

\section{Graphs as Dynamical Systems} \label{sec: rep graphs as ts}


\subsection{The \algo Representation} \label{sec: grep proof}


The key idea behind \algo is to represent vertices of a graph by the trajectory of an appropriately chosen dynamical system induced by the graph.
Such a representation has the advantage of progressively capturing more and more information about a vertex as the trajectory unfolds.
Consider a directed graph $G(V,E)$ whose vertices we want to represent (undirected graphs will be represented by having bi-directional edges between pairs of connected vertices).
We create a discrete-time dynamical system in which vertex $v$ has a state of ${\bf x}_v(t) \in \mathbb{R}^d$ at time $t\in\mathbb{N}$, for all $v\in V$, and $d$ is the dimension of the state space.
In \algo, we consider an evolution rule of the form
\begin{align}
{\bf x}_v(t+1) = \mathrm{relu}(\mathbf{W}_1 (\sum_{u\in \Gamma(v)} {\bf x}_u(t)) + \mathbf{b}_1 ) + {\bf n}_v(t+1),  \label{eq: basic filter}
\end{align}
$\forall v \in V,t\in\mathbb{N},$ where $\mathbf{W}_1\in \mathbb{R}^{d\times d}$, $\mathbf{b}_1 \in\mathbb{R}^{d\times 1}$ are trainable  parameters and relu$(x) = \max(x,0)$.
${\bf n}_v(\cdot)$ is a $d$-dimensional $(\mathbf{0}, \sigma^2 \mathbf{I})$ Gaussian noise, and $u \in \Gamma(v)$ if there is an edge from $u$ to $v$ in the graph $G$.
For any $v\in V$, starting with an initial value for ${\bf x}_v(0)$ (e.g., random or all zero) this equation 
defines a dynamical system, the (random) trajectory of which is the \algo~representation of $v$.
More generally, graphs could have features on vertices or edges (e.g.,  weights on vertices), which can be included in the evolution rule; these generalizations are outside the scope of this paper.
We use \algo($G$) to mean the set of all \algo vertex representations of $G$.

{\bf \algo is invertible.}
Our first key result is that \algo's representation allows recovery of the adjacency matrix of the graph with arbitrarily high probability. 
Here the randomness is with respect to the noise term in \algo; see~\eqref{eq: basic filter}. 
In Appendix~\ref{sec: g2s equals s2g proof}, we prove:
\begin{thm} \label{thm: g2s equals s2g}
For any directed graph $G$ and associated (random) representation \algo($G$) with sequence length $t$, there exists an inference procedure (with time complexity polynomial in  $t$) that produces an estimate $\hat{G}_t$ such that   $\lim_{t\rightarrow\infty} \mathbb{P}[G \neq \hat{G}_t] = 0$.
\end{thm}

Note that there are many ways to represent a graph that are lossless. 
For example, we can simply output the adjacency matrix row by row.
However such representations depend on assigning {\em labels} or identifiers to vertices, which would cause downstream deep learning algorithms to memorize the label structure and not generalize to other graphs. 
\algo's key property is that it is does not depend on a labeling of the graph.
Theorem~\ref{thm: g2s equals s2g} is particularly significant (despite the representation being infinite dimensional) since it shows lossless label-independent vertex representations are possible. 
To our understanding this is the first result making such a connection.
Next, we show that the noise term in \algo's evolution rule (\eqref{eq: basic filter}) is crucial for Theorem~\ref{thm: g2s equals s2g} to hold (proof in Appendix~\ref{sec: deterministic rule proof}). 

\begin{prop} \label{prop: deterministic rule}
Under any deterministic evolution rule of the form in~\eqref{eq: basic filter}, there exists a graph $G$ which cannot be reconstructed exactly from its \algo~representation with arbitrarily high probability.
\end{prop}
The astute reader might observe that invertible, label-independent representations of graphs can be used to solve the graph isomorphism problem~\cite{schrijver2003combinatorial}. However, Proposition~\ref{prop: deterministic rule} shows that \algo cannot solve the graph isomorphism problem, as that would require a {\em deterministic} representation.
Noise is necessary to {\em break symmetry} in the otherwise deterministic dynamical system. 
 Observe that the time-series for a vertex in~\eqref{eq: basic filter} depends {\em only} on the time-series of its neighboring nodes, not any explicit vertex identifiers. As a result, two graphs for which all vertices have exactly the same neighbors will have exactly the same representations, even though they may be structurally different.  The proof of Proposition~\ref{prop: deterministic rule} illustrates this phenomenon for regular graphs.

\subsection{Formal Computation Model} 
\label{sec: formal model}
Although \algo~is an invertible representation of a graph, it is unclear how it compares to other GCNN representations in the literature. Below we define a formal computational model on graphs, called \model, that includes \algo~as well as a large class of GCNN representations  in the literature.
Abstracting different representations into a formal computational model allows us reason about the fundamental limits of these methods.
We show that GCNNs with a fixed number of convolutional steps cannot compute certain functions over graphs, where a sequence-based representation such as \algo~is able to do so.  For simplicity of notation, we consider undirected graphs in this section and in the rest of this paper.


{\bf \model~model}.  Consider an undirected graph $G(V,E)$ on which we seek to compute a function $f: \mathcal{G} \rightarrow \mathbb{R}$, where $\mathcal{G}$ is the space of all undirected graphs.  
In the $k$-\model~model, computations proceed in two rounds: 
In the {\em local step}, each vertex $v$ computes a representation $r(v)$ that depends only on the subgraph of vertices that are at a distance of at most $k$ from $v$. 
Following this, in the {\em gather step}, the function $f$ is computed by applying another function $g(\cdot)$ over the collection of vertex representations $\{ r(v) : v \in V \}$.
\algo~is an instance of the $\infty$-\model~model. 
GCNNs that use localized filters with a global aggregation (e.g., \citet{kipf2016semi, dai2017learning}) also fit this model (proof in Appendix~\ref{sec: gcnn belong to model proof}).
\begin{prop} \label{ex: gcnn belong to model}
The spectral GCNN representation  in~\citet{kipf2016semi} and the spatial GCNN representation in \cite{dai2017learning}  belong to the $4$-\model~model.
\end{prop}

{\bf Fixed-length representations are insufficient}. 
We show below that for a fixed $k>0$, no algorithm from the $k$-\model~model can compute certain 
canonical graph functions exactly (proof in Appendix\ref{sec: fixed length insuf proof}).

\begin{thm} \label{thm: fixed length insuf}
For any fixed $k>0$, there exists a function $f(\cdot)$ and an input graph instance $G$ such that no $k$-\model~algorithm can compute $f(G)$ exactly.
\end{thm}

For the graph $G$ and function $f(\cdot)$ used in the proof of Theorem~\ref{thm: fixed length insuf}, we present a sequence-based representation (from the $\infty$-\model) in Appendix~\ref{sec: seq heur trees} that is able to asymptotically compute $f(G)$. 
This example demonstrates that sequence-based representations are more powerful than fixed-length graph representations in the \model~model. Further, it illustrates how a trained neural network can produce sequential representations that can be used to compute specific functions.

 {\bf \algo and graph kernels.} 
Graph kernels~\citep{yanardag2015deep, vishwanathan2010graph, kondor2016multiscale} are another popular method of representing graphs. 
The main idea here is to define or learn a vocabulary of substructures (e.g., graphlets, paths, subtrees), and use counts of these substructures in a graph as its representation. 
The Weisfeiler-Lehman (WL) graph kernel~\citep{shervashidze2011weisfeiler, weisfeiler1968reduction} is closest to \algo. 
Starting with `labels' (e.g., vertex degree) on vertices, the WL kernel iteratively performs local label updates similar to~\eqref{eq: basic filter} but typically using discrete functions/maps. 
The final representation consists of counts of these labels (i.e., a histogram) in the graph. 
Each label corresponds to a unique subtree pattern.
However, the labels themselves are not part of any structured space, and cannot be used to compare the similarity of the subtrees they represent. 
Therefore, during testing if new unseen labels (or equivalently subtrees) are encountered the resulting representation may not generalize.

\section{Neural Network Design} \label{sec: Neural Network Design}

We consider a reinforcement learning (RL) formulation for combinatorial optimization problems on graphs. RL is well-suited to such problems since the true  `labels' (i.e., the optimal solution) may be unavailable or hard to compute. 
Additionally, the objective functions in these problems can be used as natural reward signals.
An RL approach has been explored in recent works~\citet{vinyals2015pointer, bello2016neural, dai2017learning} under different representation techniques.  
Our learning framework uses the \algo~representation.
Fig.~\ref{fig: neural arch} shows our neural network architecture. 
We feed the trajectories output by \algo for all the vertices into a recurrent neural network (specifically RNN-GRU), whose outputs are then aggregated by a feedforward network to select a vertex. 
Henceforth we call this the \gtos neural network architecture. 
The key feature of this design is that the length of the sequential representation is {\em not} fixed; we vary it depending on the input instance. 
We show that our model is able to learn rules---for both generating the sequence and processing it with the RNN---that generalize to operate on long sequences.
In turn, this translates to algorithmic solutions that scale to large graph sizes.



\noindent {\bf Reinforcement learning model.} We consider a RL formulation in which vertices are chosen {\em one at a time}. 
Each time the RL agent chooses a vertex, it receives a reward.
The goal of training is to learn a policy such that cumulative reward is maximized.
We use  $Q$-learning to train the network.
\begin{wrapfigure}{l}{0.55\textwidth} 
\vspace{0pt}
  \begin{center}
    \includegraphics[width=0.45\textwidth]{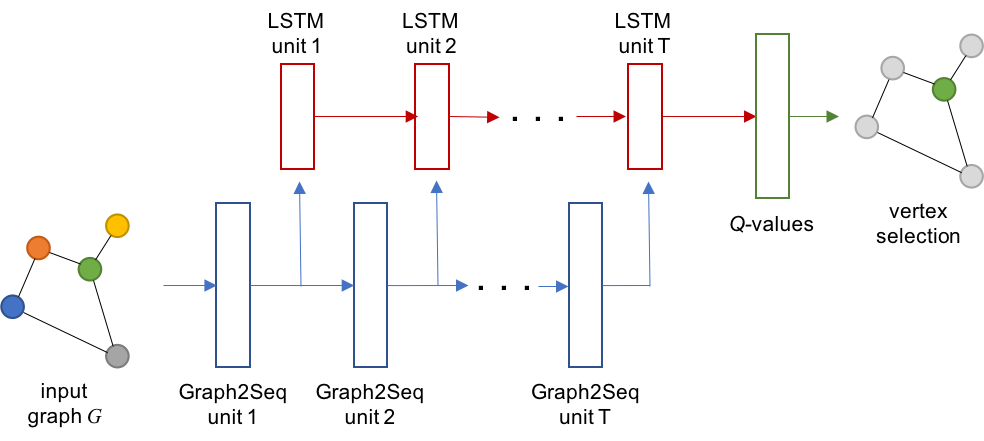}
    \caption{\small The \gtos neural network architecture. The GRU units process the representation output by \algo, operating independently on each vertex's time-series.}
    \label{fig: neural arch}
  \end{center}
  \vspace{-12pt}
\end{wrapfigure} 
For input graph instance $G(V,E)$, a subset $S\subseteq V$ and $a\in V\backslash S$, this involves using a neural network to approximate a $Q$-{\em function} $Q(G,S,a)$. Here $S$ represents the set of vertices already picked.
The neural network comprises of three modules:
(1) {\em Graph2Seq}, that takes as input the graph $G$ and set $S$ of vertices chosen so far. It generates a sequence of vectors as output for each vertex.
(2) {\em Seq2Vec} reads the sequences output of \algo and summarizes it into one vector per vertex.
(3) $Q$-{\em Network} takes the vector summary of each vertex $a \in V$ and outputs the estimated $Q(G,S,a)$ value. 
The overall architecture is illustrated in Fig.~\ref{fig: neural arch}.
To make the network practical, we truncate the sequence outputs of \algom~to a length of $T$.
However the value of $T$ is not fixed, and is varied both during training and testing according to the size and complexity of the graph instances encountered; see \S~\ref{sec: evals} for details.
We describe each module below.

{\bf Graph2Seq.}
Let $\mathbf{x}_v(t)$ denote the state of vertex $v$ and $c_v(t)$ denote the binary variable that is one if $v \in S$ and zero otherwise, at time-step $t$ in the \algo evolution.
Then, the trajectory of each vertex $v\in V$ evolves as
$ \mathbf{x}_v(t+1) = \mathrm{relu} (
 \mathbf{W}_{G,1} \sum_{u\in\Gamma(v)}\mathbf{x}_u(t) + \mathbf{w}_{G,2} c_v(t) + \mathbf{w}_{G,3} ),$ 
for $t=0,1,\ldots, T-1$. 
$\mathbf{W}_{G,1}\in\mathbb{R}^{d\times d}, \mathbf{w}_{G,2}\in\mathbb{R}^d, \mathbf{w}_{G,3}\in\mathbb{R}^d$ are trainable parameters, and $\mathbf{x}_v(0)$ is initialized to all-zeros for each $v\in V$.

{\bf Seq2Vec and $Q$-Network.}
The sequences $\{ (\mathbf{x}_v(t))^T_{t=1}: v \in V\} $ are processed by GRU units~\citep{chung2014empirical} at the vertices. 
At time-step $t$, the recurrent unit at vertex $v$ computes as input a function that depends on (i) $\mathbf{x}_v(t)$, the embedding for node $v$, (ii) $\sum_{u\in\Gamma(v)}\mathbf{x}_u(t)$, the neighboring node embeddings and (iii) $\sum_{u\in V}\mathbf{x}_u(t)$, a summary of embeddings of all nodes. 
This input is combined with the GRU's cell state $\mathbf{y}_v(t-1)$ to produce an updated cell state $\mathbf{y}_v(t)$.
The cell state at the final time-step $\mathbf{y}_v(T)$ is the desired vector summary of the \algo~sequence, and is fed as input to the $Q$-network.
We refer to Appendix~\ref{sec: net arch appendix} for equations on Seq2Vec. 
The $Q$-values are estimated as 
\begin{align}
\tilde{Q}(G,S,v) = &\mathbf{w}_{Q,1}^T \mathrm{relu} ( \mathbf{W}_{Q,2} \sum_{u\in V} \mathbf{y}_u(T) ) 
+ \mathbf{w}_{Q,3}^T \mathrm{relu} \left( \mathbf{W}_{Q,4} \mathbf{y}_v(T) \right), \label{eq: Q func est}
\end{align}
with $\mathbf{W}_{Q,1}, \mathbf{W}_{Q,3} \in \mathbb{R}^{d\times d}$ and $\mathbf{w}_{Q,2}, \mathbf{w}_{Q,4} \in \mathbb{R}^d$ being learnable parameters.
All transformation functions in the network leading up to~\eqref{eq: Q func est} are differentiable.
This makes the whole network differentiable, allowing us to train it end to end. 

{\bf Remark.}
In Fig.~\ref{fig: neural arch} the \algo RNN and the GRU RNN can also be thought of as two layers of a two-layer GRU. 
We have deliberately separated the two RNNs to highlight the fact that the sequence produced by \algo (blue in Fig,~\ref{fig: neural arch}) is the embedding of the graph, which is then read using a GRU layer (red in Fig.~\ref{fig: neural arch}).  
Our architecture is not unique, and other designs for generating and/or reading the sequence are possible.





\section{Evaluations} \label{sec: evals}

\begin{figure}[t]
  \centerline{\subfigure[]{\includegraphics[height=32mm]{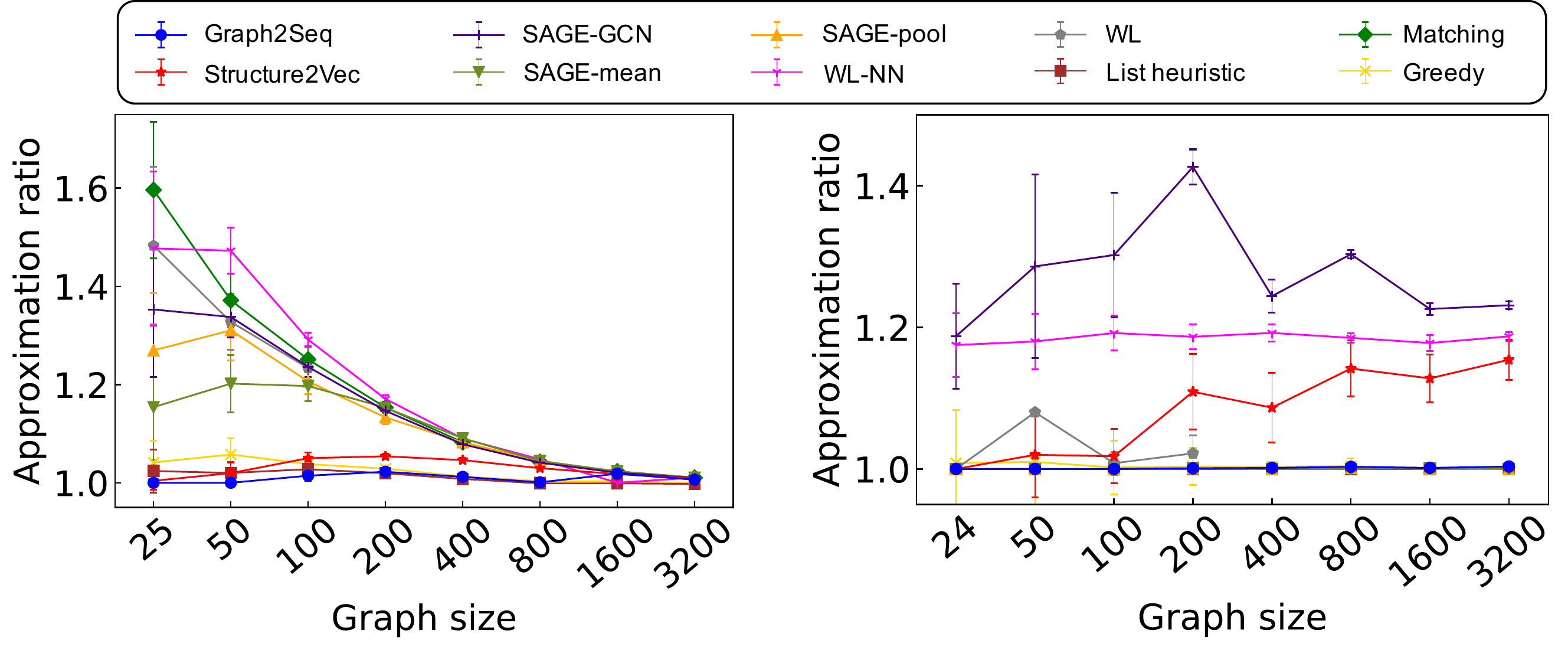}
     \label{fig: scalability}}
    \subfigure[]{
    	{\scriptsize
    	\begin{tabular}[b]{ccc}
		\hline
		Heuristic & $10k$ nodes & $25k$ nodes \\
		\hline
		Graph2Seq (ER) & 6444 & 16174 \\
		Structure2Vec (ER) & 6557 & 16439 \\
		\hline 
		Graph2Seq (bipartite) & 5026 & 12603 \\
		Structure2Vec (bipartite) & 5828 & 14409 \\
		\hline
		\smallskip 
		\smallskip
		\smallskip
		\smallskip
    	\end{tabular}
	}
    \label{fig: large table}	
    }
}
    \caption{\small Scalability of \gtos~and GCNN in (a) Erdos-Renyi graphs (left), random bipartite graphs (right), and (b) on much larger graphs. The neural networks have been trained over the same graph types as the test graphs. Error bars show one standard deviation.}
    \label{fig: scalabilityfull}
    \vspace{-10pt}
\end{figure}

In this section we present our evaluation results for \algo.
We address the following central questions: 
(1) How well does \gtos~scale?
(2) How well does \gtos~generalize to new graph types? 
(3) Can we apply \gtos~to a variety of problems? 
(4) Does adversarial training improve scalability and generalization? 
To answer these questions, we experiment with \gtos on three classical graph optimization problems: minimum vertex cover (MVC), max cut (MC) and maximum independent set (MIS). 
These are a set of problems well known to be NP-hard, and also greatly differ in their structure~\citep{williamson2011design}.
We explain the problems below. 

\noindent {\bf Minimum vertex cover.}
The MVC of a graph $G(V,E)$ is the smallest cardinality set $S\subseteq V$ such that for every edge $(u,v) \in E$ at least one of $u$ or $v$ is in $S$.
Approximation algorithms to within a factor 2 are known for MVC; however it cannot be approximated better than 1.3606 unless P$=$NP. \\
\noindent {\bf Max cut.}
In the MC problem, for an input graph instance $G(V,E)$ we seek a cut $(S, S^c)$ where $S\subseteq V$ such that the number of edges crossing the cut is maximized. 
This problem can be approximated within a factor 1.1383 of optimal, but not within 1.0684 unless P$=$NP.  \\
\noindent {\bf Maximum independent set.}
For a graph $G(V,E)$ the MIS denotes a set $S\subseteq V$ of maximum cardinality such that for any $u, v\in S$, $(u,v) \notin E$. 
The maximum independent set is complementary to the minimum vertex cover---if $S$ is a MIC of $G$, then $V \backslash S$ is the MVC. 
However, from an approximation standpoint, MIS is hard to approximate within $n^{1-\epsilon}$ for any $\epsilon > 0$, despite constant factor approximation algorithms known for MVC.

\noindent {\bf Heuristics compared.}
In each problem, we compare \gtos~against:
(1) Structure2Vec~\citep{dai2017learning}, 
(2) GraphSAGE~\citep{hamilton2017inductive} using (a) GCN, (b) mean and (c) pool aggregators, 
(3) WL kernel NN~\citep{lei2017deriving}, 
(4) WL kernel embedding, in which the feature map corresponding to WL subtree kernel of the subgraph in a 5-hop neighborhood around each vertex is used as its vertex embedding~\citep{shervashidze2011weisfeiler}.
Since we test on large graphs, instead of using a learned label lookup dictionary we use a standard hash function for label shortening at each step. 
In each of the above, the outputs of the last layer are fed to a $Q$-learning network as in \S\ref{sec: Neural Network Design}.
Unlike \gtos the depth of the above neural network (NN) models are fixed across input instances. 
We also consider the following well-known (non-learning based) heuristics for each problem:
(5) {\em Greedy algorithms},
(6) {\em List heuristic},
(7) {\em Matching heuristic}. 
We refer to Appendix~\ref{sec: heuristics} for details on these heuristics. 

We attempt to compute the optimal solution via the Gurobi optimization package~\citep{gurobi}. We run the Gurobi solver with a cut-off time of 240~s, and  
report performance in the form of approximation ratios relative to the solution found by the Gurobi solver. We do not compare against Deepwalk~\citep{perozzi2014deepwalk} or node2vec~\citep{grover2016node2vec} since these methods are designed for obtaining vertex embeddings over a single graph. 
They are inappropriate for models that need to generalize over multiple graphs. 
This is because the vertex embeddings in these approaches can be arbitrarily rotated without a consistent `alignment' across graphs. 
The number of parameters can also grow linearly with graph size. 
We refer to~\citet[Appendix D]{hamilton2017inductive} for details. 

\begin{figure*}[t]
  \centerline{
     \includegraphics[height=37mm]{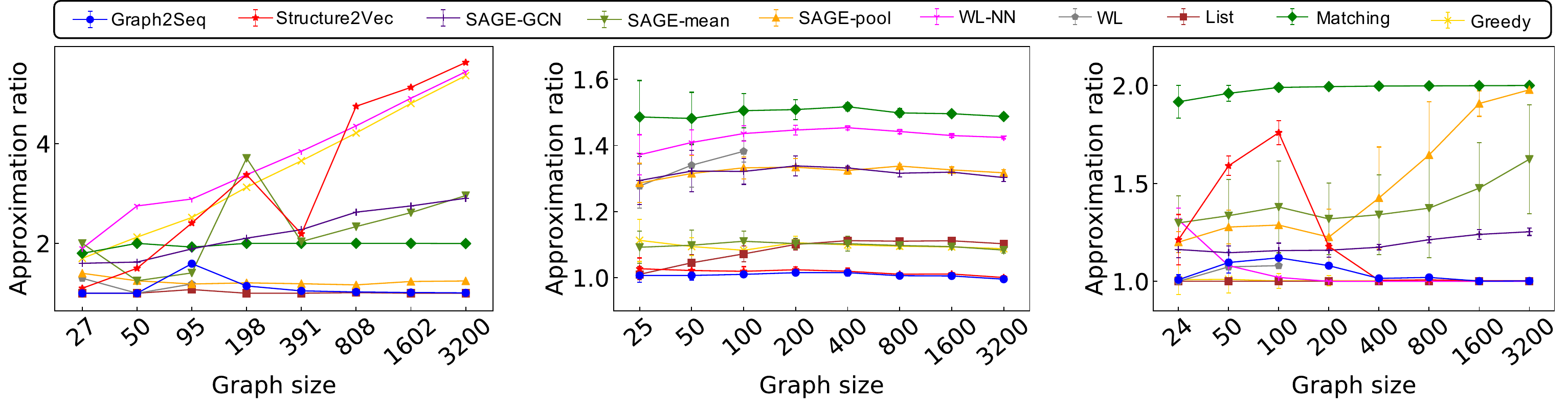}
     \label{fig: er greedy}
     }
     \vspace{-2mm}
    \caption{\small Minimum vertex cover in worst-case graphs for greedy (left), random regular graphs (center), and random bipartite graphs (right), using models trained on size-15 Erdos-Renyi graphs. Each point shows that average and standard deviation for 20 randomly sampled graphs of that type. The worst-case greedy graph type, however, has only one graph at each size; hence, the results show outliers for this graph type.}
    \label{fig: er_trained}
    \vspace{-15pt}
\end{figure*}

\noindent {\bf Training.}
We train \gtos, Structure2Vec, GraphSAGE and WL kernel NN all on small Erdos-Renyi (ER) graphs of size 15, and edge probability $0.15$. 
During training we truncate the \algo representation to 5 observations (i.e., $T=5$, see Fig.~\ref{fig: neural arch}). 
In each case, the model is trained for 100,000 iterations, except WL kernel NN which is trained for 200,000 iterations since it has more parameters. 
We use experience replay~\citep{mnih2013playing}, a learning rate of $10^{-3}$, Adam optimizer~\citep{kingma2014adam} and an exploration probability that is reduced from $1.0$ to a resting value of $0.05$ over 10,000 iterations. 
The amount of noise added in the evolution ($n_v(t)$ in Equation~\ref{eq: basic filter}) seemed to not matter; we have set the noise variance $\sigma^2$ to zero in all our experiments (training and testing).
As far as possible, we have tried to keep the hyperparameters in \gtos and all the neural network baselines to be the same.  
For example, all of the networks have a vertex embedding dimension of $16$, use the same neural architecture for the $Q$-network and Adam optimizer for training. 

\noindent {\bf Testing}.
For any $T > 0$, let $\text{\gtos}(T)$ denote the neural network (\S\ref{sec: Neural Network Design}) in which \algo is restricted to a sequence length of $T$.
To test a graph $G(V,E)$, we feed $G$ as input to the networks $\text{\gtos}(1), \text{\gtos}(2), \ldots, \text{\gtos}(T_\mathrm{max})$, and choose the `best' output as our final output.  
For each $T$, $\text{\gtos}(T)$ outputs a solution set $S_T \subseteq V$. 
The best output corresponds to that $S_T$ having the maximum objective value. 
For e.g., in the case of MVC the $S_T$ having the smallest cardinality is the best output. 
This procedure is summarized in detail in Algorithm~\ref{algo: graph2seq testing} 
in Appendix~\ref{sec: net arch appendix}.
We choose $T_\mathrm{max}=15$ in our experiments.
The time complexity of \gtos is $O(|E|T_\mathrm{max} + |V|)$.

To test generalization across graph types, we consider the following graph types: 
(1) ER graphs with edge probability $0.15$; 
(2) random regular graphs with degree $4$;  
(3) random bipartite graphs with equal sized partites and edge probability $0.75$; 
(4) worst-case examples, 
such as the worst-case graph for the greedy heuristic on MVC, which has a $O(\log n)$ approximation factor~\cite{johnson1973approximation}.
(5) Two-dimensional grid graphs, in which the sides contain equal number of vertices. 
For each type, we test on graphs with number of vertices ranging from 25--3200 in exponential increments (except WL embedding which is restricted to size 100 or 200 since it is computationally expensive).
Some of the graphs considered are dense---e.g., a 3200 node ER graph has 700,000 edges; a 3200 node random bipartite graph has 1.9 million edges.
We also test on sparse ER and bipartite graphs of sizes $10,000$ and $25000$ with an average degree of 7.5.

\subsection{Scalability and Generalization across Graph Types} \label{sec: scalability}

{\bf Scalability.}
To test scalability, we train all the NN models on small graphs of a type, and test on larger graphs of the same type. 
For each NN model, we use the {\em same} trained parameters for all of the test examples in a graph type.
We consider the MVC problem and train on: (1) size-15 ER graphs, and (2) size-20 random bipartite graphs.  
The models trained on ER graphs are then tested on ER graphs of sizes 25--3200; similarly the models trained on bipartite graphs are tested on bipartite graphs of sizes 24--3200.
We present the results in Fig.~\ref{fig: scalability}.
We have also included non-learning-based heuristics for reference.
In both ER and bipartite graphs, we observe that \gtos generalizes well to graphs of size roughly 25 through 3200, even though it was trained on size-15 graphs. 
Other NN models, however, either generalize well on only one of the two types (e.g., Structure2Vec performs well on ER graphs, but not on bipartite graphs) or do not generalize in both types. 
\gtos~generalizes well to even larger graphs.
Fig.~\ref{fig: large table} presents results of testing on size 10,000 and 25,000 ER and random bipartite graphs.
We observe the vertex cover output by \gtos~is at least 100 nodes fewer than Structure2Vec.

{\bf Generalization across graph types.}
Next we test how the models generalize to different graph types. 
We train the models on size-15 ER graphs, and test them on three graph types: (i) worst-case graphs, (ii) random regular graphs, and (iii) random bipartite graphs. 
For each graph type, we vary the graph size from 25 to 3200 as before. 
Fig.~\ref{fig: er_trained} plots results for the different baselines. 
In general, \gtos~has a performance that is within 10\% of the optimal, across the range of graph types and sizes considered.
The other NN baselines demonstrate behavior that is not consistent and have certain classes of graph types/sizes where they perform poorly.


{\bf Adversarial training.}
We also trained \gtos on a certain class of adversarial `hard' examples for minimum vertex cover, and observed further improvements in generalization.
We refer to Appendix~\ref{apx: eval} for details and results of this method. 

\subsection{Other Problems: MC and MIS} \label{sec: mc and mis}

We test and compare \gtos~on the MC and MIS problems. 
As in MVC, our results demonstrate consistently good scalability and generalization of \gtos~across graph types and sizes.
As before, we train the NN models on size-15 ER graphs ($p=0.15$) and test on different graphs. 

\noindent {\bf Max cut.}
We test on 
(1) ER graphs, and (2) two-dimensional grid graphs. 
For each graph type, we vary the number of vertices in the range $25$--$3200$, and use the same trained model for all of the tests. 
The results of our tests are presented in Fig.~\ref{fig: maxcut}.
We notice that for both graph types \gtos~is able to achieve an approximation less that $1.04$ times the (timed) integer program output.   
\begin{figure}[t]
  \centerline{
    \subfigure[]{\includegraphics[height=28mm]{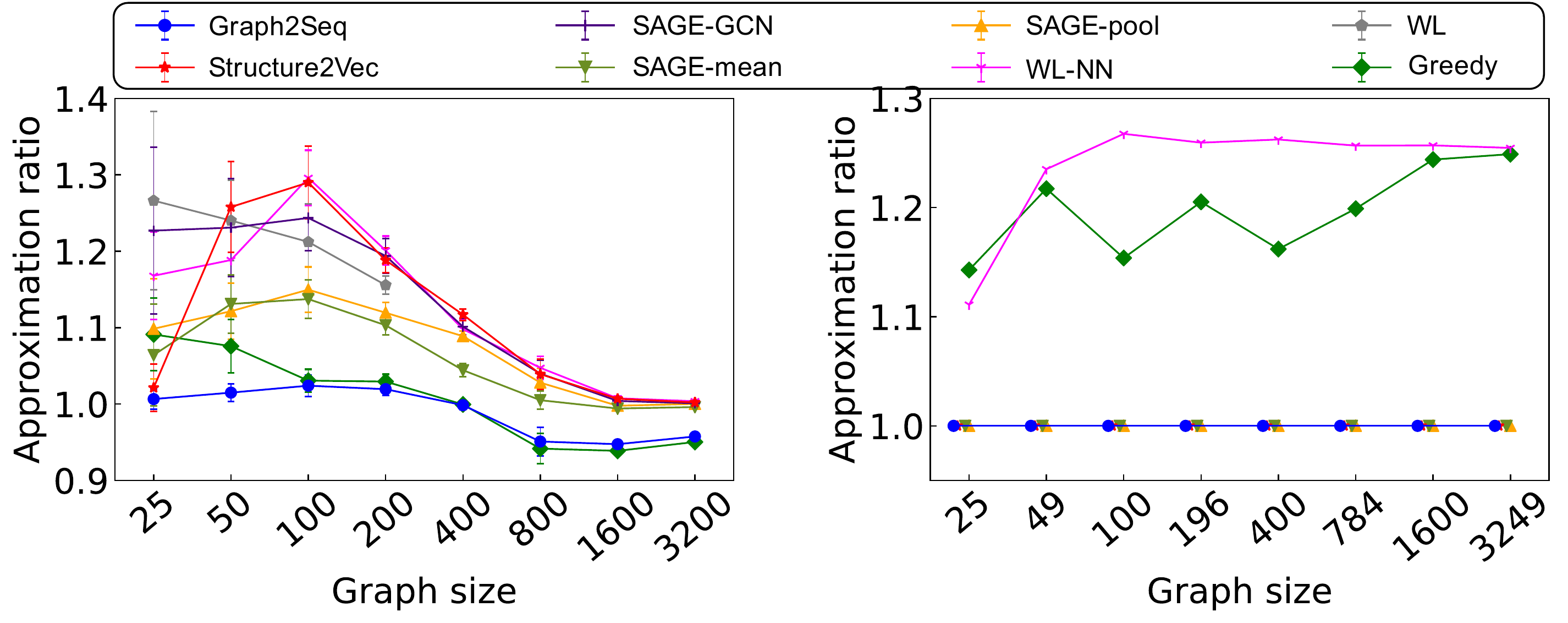}
     \label{fig: maxcut}}
    \subfigure[]{\includegraphics[height=28mm]{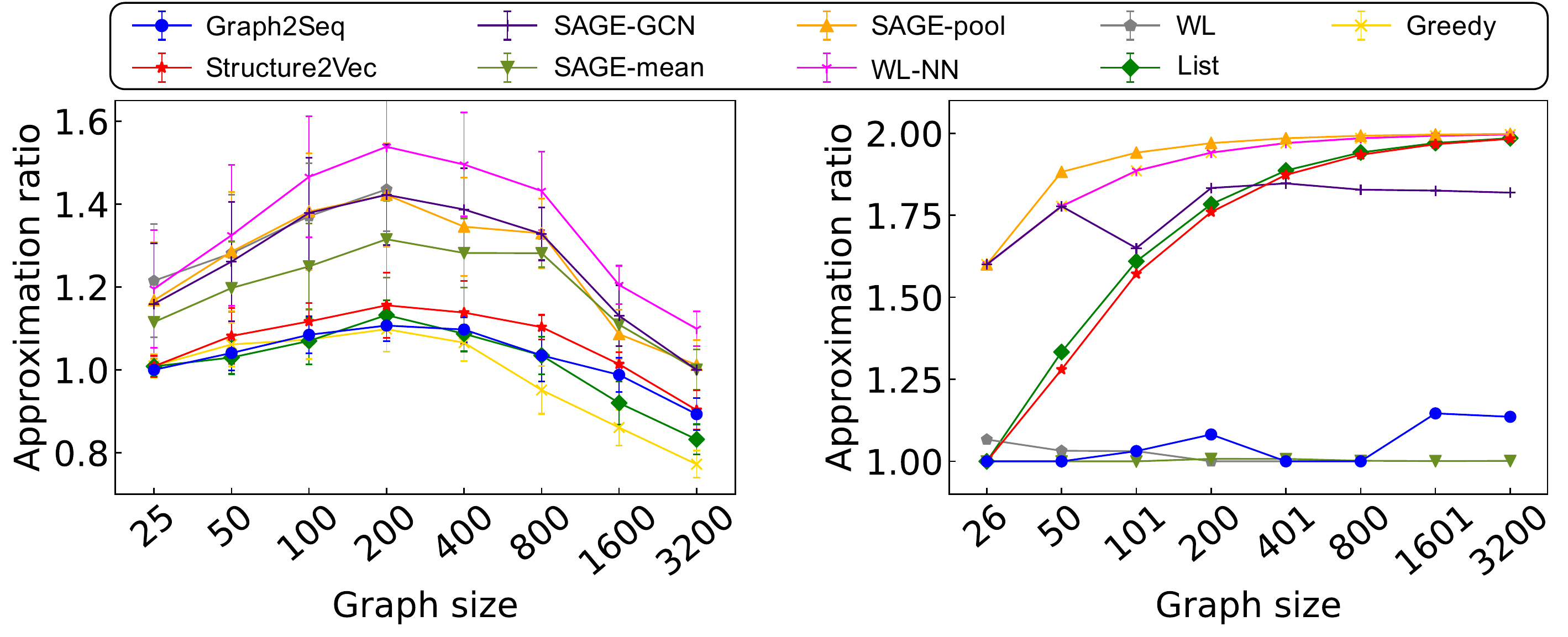}
     \label{fig: indset}}
    }
    \caption{\small (a) Max cut in Erdos-Renyi graphs (left), Grid graphs (right). (b) Maximum independent set in Erdos-Renyi graphs (left), structured bipartite graphs (right). Some approximation ratios are less than 1 due to the time cut-off of the Gurobi solver.}
    \label{fig: mcandindset}
      \vspace{-18pt}
\end{figure}
\\ \noindent {\bf Maximum independent set.}
We test on (1) ER graphs, and (2) worst-case bipartite graphs for the greedy heuristic.
The number of vertices is varied in the range 25--3200 for each graph type. 
We present our results in Fig.~\ref{fig: indset}.
In ER graphs, \gtos~shows a reasonable consistency in which it is always less than 1.10 times the (timed) integer program solution. 
In the bipartite graph case we see a performance within 8\% of optimal across all sizes. 

\section{Conclusion}


We proposed \algo~ that represents vertices of graphs as infinite time-series of vectors. The representation melds naturally with modern RNN architectures that take time-series as inputs. We applied this combination to three canonical combinatorial optimization problems on graphs, ranging across the complexity-theoretic hardness spectrum. Our empirical results best state-of-the-art approximation algorithms for these problems on a variety of graph sizes and types. In particular, \algo exhibits significantly better scalability and generalization than existing GCNN representations in the literature.  An open direction involves a more systematic study of the capabilities of \algo across the panoply of graph combinatorial optimization problems, as well as its performance in concrete (and myriad) downstream applications. 
Another open direction involves interpreting the policies learned by \algo to solve specific combinatorial optimization problems (e.g., as in LIME \citep{ribeiro2016should}). 
A detailed analysis of the \algo dynamical system to study the effects of sequence length on the representation is also an important direction.


\bibliography{mybib}
\bibliographystyle{iclr2019_conference}

\newpage
\appendix
\section{Background: Graph Convolutional Neural Networks} \label{sec: background}

An ideal graph representation is one that captures all innate structures of the graph relevant to the task at hand, and moreover can also be learned via gradient descent methods.  
However, this is challenging since the relevant structures could range anywhere from local attributes (example: node degrees) to long-range dependencies spanning across a large portion of the graph (example: does there exist a path between two vertices)~\citep{kuhn2016local}.
Such broad scale variation is also a well-known issue in computer vision (image classification, segmentation etc.), wherein convolutional neural network (CNN) designs have been used quite successfully~\citep{krizhevsky2012imagenet}.
Perhaps motivated by this success, recent research has focused on generalizing the traditional CNN architecture to develop designs for graph convolutional neural networks (GCNN)~\citep{bruna2013spectral, niepert2016learning}.
By likening the relationship between adjacent pixels of an image to that of adjacent nodes in a graph, the GCNN seeks to emulate CNNs by defining localized `filters' with shared parameters.

Current GCNN filter designs can be classified into one of two categories: {\em spatial}~\citep{kipf2016semi, dai2017learning, nowak2017note}, and {\em spectral}~\citep{defferrard2016convolutional}.
For an integral hyper-parameter $K\geq 0$, filters in either category process information from a $K$-local neighborhood surrounding a node to compute the output. Here we consider localized spectral filters such as proposed in~\citet{defferrard2016convolutional}.
The difference between the spatial and spectral versions arises in the precise way in which the aggregated local information is combined.

\noindent {\em Spatial GCNN}. For input feature vector $\mathbf{x}_v$ at each node $v\in V$ of a graph $G$, a spatial filtering operation is the following:
\begin{align}
\mathbf{y}_v = \sigma\left( \sum_{k=0}^{K-1} \mathbf{W}_k \left(\sum_{u\in V} (\tilde{\mathbf{A}}^k)_{v,u} \mathbf{x}_u\right) + \mathbf{b}_0 \right) \quad \forall v \in V,  \label{eq: spatial update}
\end{align}
where $\mathbf{y}_v$ is the filter output, ${\bf W}_k, k=1,\ldots,K$ and ${\bf b}_0$ are learnable parameters, and $\sigma$ is a non-linear activation function that is applied element-wise.
$\tilde{\mathbf{A}} = \mathbf{D}^{-1/2}\mathbf{A}\mathbf{D}^{-1/2}$ is the normalized adjacency matrix, and $\mathbf{D}$ is the diagonal matrix of vertex degrees.
Use of un-normalized adjacency matrix is also common. 
The $k$-power of the adjacency matrix selects nodes a distance of at most $k$ hops from $v$.
ReLU is a common choice for $\sigma$.
We highlight two aspects of spatial GCNNs: (i) the feature vectors are aggregated from neighboring nodes directly specified through the graph topology, and (ii) the aggregated features are summarized via an addition operation.

\noindent {\em Spectral GCNN}. Spectral GCNNs  use the notion of graph Fourier transforms to define convolution operation as the inverse transform of multiplicative filtering in the Fourier domain.
Since this is a non-local operation potentially involving data across the entire graph, and moreover it is computationally expensive to compute the transforms, recent work has focused on approximations to produce a local spectral filter of the form
\begin{align}
\mathbf{y}_v = \sigma \left( \sum_{k=0}^{K-1} \mathbf{W}'_k \left( \sum_{u \in V}  (\tilde{{\bf L}}^k)_{v,u} \mathbf{x}_u \right) + \mathbf{b}'_0 \right) \quad \forall v \in V,  \label{eq: spectral update}
\end{align}
where $\tilde{{\bf L}} = \mathbf{I} - \tilde{\mathbf{A}}$ is the normalized Laplacian of the graph, $(\tilde{{\bf L}}^k)_{v,u}$ denotes the entry at the row corresponding to vertex $v$ and column corresponding to vertex $u$ in $\tilde{{\bf L}}^k$, and ${\bf W}'_k, {\bf b}'_0$ are parameters~\citep{defferrard2016convolutional, kipf2016semi}.
As in the spatial case, definitions using unnormalized version of Laplacian matrix are also used. 
$\sigma$ is typically the identity function here.
The function in~\eqref{eq: spectral update} is a local operation because the $k$-th power of the Laplacian, at any row $v$, has a support no larger than the $k$-hop neighborhood of $v$.
Thus, while the aggregation is still localized, the feature vectors are now weighted by the entries of the Laplacian before summation.

\noindent {\em Spectral and Spatial GCNN are equivalent}. The distinction between spatial and spectral convolution designs is typically made owing to their seemingly different definitions.
However we  show that both designs are mathematically equivalent in terms of their representation capabilities.
\begin{prop}
Consider spatial and spectral filters in~\eqref{eq: spatial update} and~\eqref{eq: spectral update}, using 
the same nonlinear activation function $\sigma$ and $K$.
Then, for graph $G(V,E)$, for any choice of parameters ${\bf W}_k$ and ${\bf b}_0$ for $k=1,\ldots,K$ there exists parameters $\mathbf{W}'_k$ and $\mathbf{b}'_0$ such that the filters represent the same transformation function, and vice-versa.
\end{prop}
\begin{proof}
Consider a vertex set $V = \{1,2,\ldots,n\}$ and $d$-dimensional vertex states $\mathbf{x}_i$ and $\mathbf{y}_i$ at vertex $i\in V$.
Let $\mathbf{X} = [\mathbf{x}_1,\ldots,\mathbf{x}_n]$ and $\mathbf{Y} = [\mathbf{y}_1,\ldots,\mathbf{y}_n]$ be the matrices obtained by concatenating the state vectors of all vertices.
Then the spatial transformation function of~\eqref{eq: spatial update} can be written as
\begin{align}
\mathbf{Y} = \sigma\left( \sum_{k=0}^{K-1} \mathbf{W}_k \mathbf{X} \tilde{\mathbf{A}}^k + \mathbf{b}_0 \mathbf{1}^T \right), \label{eq: spatial mat version}
\end{align}
while the spectral transformation function of~\eqref{eq: spectral update} can be written as
\begin{align}
\mathbf{Y} &= \sigma \left( \sum_{k=0}^{K-1} \mathbf{W}'_k \mathbf{X} \tilde{\mathbf{L}}^k + \mathbf{b}'_0 \mathbf{1}^T \right) \\
&=  \sigma \left( \sum_{k=0}^{K-1} \mathbf{W}'_k \mathbf{X} (\mathbf{I} - \tilde{\mathbf{A}})^k + \mathbf{b}'_0 \mathbf{1}^T \right) \label{eq: lap defn} \\
&=  \sigma \left( \sum_{k=0}^{K-1} \mathbf{W}'_k \mathbf{X} \sum_{i=0}^k \binom{k}{i}(-1)^{k-i} \tilde{\mathbf{A}}^i + \mathbf{b}'_0 \mathbf{1}^T \right) \label{eq: mat binom thm} \\
&= \sigma \left( \sum_{k=0}^{K-1} \left( \sum_{i=k}^{K-1} \mathbf{W}'_i \binom{i}{k}(-1)^{i-k} \right) \mathbf{X} \tilde{\mathbf{A}}^k +  \mathbf{b}'_0 \mathbf{1}^T  \right). \label{eq: spectral mat version}
\end{align}
The~\eqref{eq: lap defn} follows by the definition of the normalized Laplacian matrix, and~\eqref{eq: mat binom thm} derives from binomial expansion.
To make the transformation in~\eqref{eq: spatial mat version} and~\eqref{eq: spectral mat version} equal, we can set
\begin{align}
\sum_{i=k}^{K-1} \mathbf{W}'_i \binom{i}{k}(-1)^{i-k} = \mathbf{W}_k, \quad \forall ~0\leq k\leq K-1,
\end{align}
and check if there are any feasible solutions for the primed quantities.
Clearly there are, with one possible solution being  $\mathbf{b}'_0 = \mathbf{b}_0$ and
\begin{align}
\mathbf{W}'_{K-1} &= \mathbf{W}_{K-1} \\
\mathbf{W}'_{k} &= \mathbf{W}_k - \sum_{i=k+1}^{K-1} \mathbf{W}'_i \binom{i}{k} (-1)^{i-k},  
\end{align}
$\forall ~0 \leq k \leq K-2.$
Thus for any choice of values for $\mathbf{W}_k, \mathbf{b}_0$ for $k=0,\ldots, K-1$ there exists $\mathbf{W}'_k, \mathbf{b}'_0$ for $k=0,\ldots,K-1$ such that the spatial and spectral transformation functions are equivalent.
The other direction (when $\mathbf{W}'_k$ and $\mathbf{b}_0$ are fixed), is similar and straightforward.
\end{proof}

Depending on the application, the convolutional layers may be supplemented with pooling and coarsening layers that summarize outputs of nearby convolutional filters to form a progressively more compact spatial representation of the graph.
This is useful in classification tasks where the desired output is one out of a few possible classes~\citep{bruna2013spectral}.
For applications requiring decisions at a per-node level (e.g. community detection), a popular strategy is to have multiple repeated convolutional layers that compute vector representations for each node, which are then processed to make a decision~\citep{dai2017learning, bruna2017community, nowak2017note}.
The conventional wisdom here is to have as many layers as the diameter of the graph, since filters at each layer aggregate information only from nearby nodes.
Such a strategy is sometimes compared to the message passing algorithm~\citep{gilmer2017neural}, though the formal connections are not clear as noted in~\citet{nowak2017note}.
Finally the GCNNs described so far are all end-to-end differentiable and can be trained using mainstream techniques for supervised, semi-supervised or reinforcement learning applications.

Other lines of work use ideas inspired from word embeddings for graph representation~\citep{grover2016node2vec, perozzi2014deepwalk}. Post-GCNN representation,  LSTM-RNNs have been used to analyze time-series data structured over a graph.
\citet{seo2016structured} propose a model which combines GCNN and RNN to predict moving MNIST data.
~\citet{liang2016semantic} design a graph LSTM for semantic object parsing in images.

\section{Section~\ref{sec: rep graphs as ts} Proofs}

\subsection{Proof of Theorem~\ref{thm: g2s equals s2g}} \label{sec: g2s equals s2g proof}
\begin{proof}
Consider a \algo~trajectory on graph $G(V,E)$ according to~\eqref{eq: basic filter} in which the vertex states are initialized randomly from some distribution.
Let $\mathbf{X}_v(t)$ (resp.\ $\mathbf{x}_v(t)$) denote the random variable (resp.\ realization) corresponding to the state of vertex $v$ at time $t$.
For time $T>0$ and a set $S\subseteq V$, let $\mathbf{X}^T_S$ denote the collection of random variables $ \{ \mathbf{X}_v(t): v\in S, 0 \leq t \leq T \}$; $\mathbf{x}_V^T$ will denote the realizations.

An information theoretic estimator to output the graph structure by looking at the trajectory $\mathbf{X}_V^T$ is the directed information graph considered in~\citet{quinn2015directed}.
Roughly speaking, the estimator evaluates the conditional directed information for every pair of vertices $u,v \in V$, and declares an edge only if it is positive (see Definition 3.4 in~\citet{quinn2015directed} for details).  Estimating  conditional directed information efficiently from samples is itself an active area of research \cite{quinn2011estimating}, but simple plug-in estimators with a standard kernel density estimator will be consistent. Since the theorem statement did not specify sample efficiency (i.e., how far down the trajectory do we have to go before estimating the graph with a required probability), the inference algorithm is simple and polynomial in the length of the trajectory. The key question is whether the directed information graph is indeed the same as the underlying graph $G$.
Under some conditions on the graph dynamics (discussed below in Properties~\ref{prop: one}--\ref{prop: three}), this holds and it suffices for us to show that the dynamics generated according to~\eqref{eq: basic filter} satisfies those conditions.

\begin{ppt} \label{prop: one}
For any $T>0$, $P_{\mathbf{X}_V^T}(\mathbf{x}_V^T) > 0$ for all $\mathbf{x}_V^T$.
\end{ppt}
 This is a technical condition that is required to avoid degeneracies that may arise in deterministic systems.
 Clearly \algo's dynamics satisfies this property due to the additive i.i.d.\ noise in the transformation functions.

\begin{ppt} \label{prop: two}
The dynamics is strictly causal, that is $P_{\mathbf{X}_V^T}(\mathbf{x}_V^T)$ factorizes as $\prod_{t=0}^T \prod_{v\in V} P_{\mathbf{X}_v(t)|\mathbf{X}^{t-1}_{V}}(\mathbf{x}_v(t) | \mathbf{x}_V^{t-1})$.
\end{ppt}
This is another technical condition that is readily seen to be true for \algo.
The proof also follows from Lemma 3.1 in~\citet{quinn2015directed}.

\begin{ppt} \label{prop: three}
$G$ is the minimal generative model graph for the random processes $\mathbf{X}_v(t), v\in V$.
\end{ppt}
Notice that the transformation operation~\eqref{eq: basic filter} in our graph causes $\mathbf{X}_V^T$ to factorize as
\begin{align}
P_{\mathbf{X}_V^T}(\mathbf{x}_V^T) = \prod_{t=0}^T \prod_{v\in V} P_{\mathbf{X}_v(t)|\mathbf{X}^{t-1}_{\Gamma(v)}}(\mathbf{x}_v(t) | \mathbf{x}_{\Gamma(v)}^{t-1}) \label{eq: G factor}
\end{align}
for any $T>0$, where $\Gamma(v)$ is the set of neighboring vertices of $v$ in $G$.
Now consider any other graph $G'(V, E')$.
$G'$ will be called a minimal generative model for the random processes $\{\mathbf{X}_v(t): v\in V, t \geq 0 \}$ if \\
(1) there exists an alternative factorization of $P_{\mathbf{X}_V^T}(\mathbf{x}_V^T)$ as
\begin{align}
P_{\mathbf{X}_V^T}(\mathbf{x}_V^T) = \prod_{t=0}^T \prod_{v\in V} P_{\mathbf{X}_v(t)|\mathbf{X}^{t-1}_{\Gamma'(v)}}(\mathbf{x}_v(t) | \mathbf{x}_{\Gamma'(v)}^{t-1}) \label{eq: G prime factor}
\end{align}
for any $T>0$, where $\Gamma'(v)$ is the set of neighbors of $v$ in $G'$, and  \\
(2) there does not exist any other graph $G''(V, E'')$ with $E'' \subset E$ and a factorization of  $P_{\mathbf{X}_V^T}(\mathbf{x}_V^T)$ as  $\prod_{t=0}^T \prod_{v\in V} P_{\mathbf{X}_v(t)|\mathbf{X}^{t-1}_{\Gamma''(v)}}(\mathbf{x}_v(t) | \mathbf{x}_{\Gamma''(v)}^{t-1})$ for any $T>0$, where $\Gamma''(v)$ is the set of neighbors of $v$ in $G''$.

Intuitively, a minimal generative model is the smallest spanning graph that can generate the observed dynamics.
To show that $G(V, E)$ is indeed a minimal generative model, let us suppose the contrary and assume there exists another graph $G'(V, E')$ with $E' \subset E$ and a factorization of $P_{\mathbf{X}_V^T}(\mathbf{x}_V^T)$  as in~\eqref{eq: G prime factor}.
In particular, let $v$ be any node such that $\Gamma'(v) \subset \Gamma(v)$.
Then by marginalizing the right hand sides of~\eqref{eq: G factor} and~\eqref{eq: G prime factor}, we get
\begin{align}
P_{\mathbf{X}_v(1)|\mathbf{X}^{0}_{\Gamma(v)}}(\mathbf{x}_v(1) | \mathbf{x}_{\Gamma(v)}^{0}) = P_{\mathbf{X}_v(1)|\mathbf{X}^{0}_{\Gamma'(v)}}(\mathbf{x}_v(1) | \mathbf{x}_{\Gamma'(v)}^{0}). \label{eq: ad absurdum eq}
\end{align}
Note that~\eqref{eq: ad absurdum eq} needs to hold for all possible realizations of the random variables $\mathbf{X}_v(1), \mathbf{X}_{\Gamma(v)}^0$ and $\mathbf{X}_{\Gamma'(0)}^0$.
However if the parameters $\mathbf{\Theta}_0$ and $\Theta_1$ in~\eqref{eq: basic filter}  are generic, this is clearly not true.
To see this, let $u \in \Gamma(v)\backslash \Gamma'(v)$ be any vertex.
By fixing the values of $\mathbf{x}_v(1), \mathbf{x}_{\Gamma(v)\backslash \{ u\}}^0$ it is possible to find two values for $\mathbf{x}_u(0)$, say $\mathbf{a_1}$ and $\mathbf{a}_2$, such that
 \begin{align}
  \mathrm{ReLU}\left( \mathbf{\Theta}_0 \left(\sum_{ i \in \Gamma(v)\backslash \{u\} } \mathbf{x}_i(0) + \mathbf{a}_1 \right) + \Theta_1 \right) \notag \\
  \neq  \mathrm{ReLU}\left( \mathbf{\Theta}_0 \left(\sum_{ i \in \Gamma(v)\backslash \{u\} } \mathbf{x}_i(0) + \mathbf{a}_2 \right) + \Theta_1 \right).
 \end{align}
 As such the Gaussian distributions in these two cases will have different means.
 However the right hand side Equation~\eqref{eq: ad absurdum eq} does not depend on $\mathbf{x}_u(0)$ at all, resulting in a contradiction.
 Thus $G$ is a minimal generating function of $\{\mathbf{X}_v(t): v\in V, t\geq 0 \}$.
 Thus Property~\ref{prop: three} holds as well.
Now the result follows from the following Theorem.

\begin{thm}[Theorem 3.6,~\citet{quinn2015directed}]
If Properties~\ref{prop: one},~\ref{prop: two} and~\ref{prop: three} are satisfied, then the directed information graph is equivalent to the graph $G$.
\end{thm}
\end{proof}

\subsection{Proof of Proposition~\ref{prop: deterministic rule}} \label{sec: deterministic rule proof}
\begin{proof}
Consider 4-regular graphs $R_1$ and $R_2$ with vertices $\{0,1,\ldots,7\}$ and edges $\{(0,3), (0,5), (0,6), (0,7), (1,2), (1,4), (1,6), (1,7), \allowbreak (2,3), (2,5), (2,6), (3,4), (3,5), (4,5), \allowbreak (4,7), (6,7)\}$ and $\{(0,1), (0,2), (0,4), (0,7), (1,4), (1,5), (1,6), (2,3), \allowbreak (2,4), (2,7), (3,5), (3,6), \allowbreak (3,7),  (4,6), (5,6), (5,7) \}$  respectively.
Then under a deterministic evolution rule, since $R_1$ and $R_2$ are 4-regular graphs, the trajectory will be identical at all nodes across the two graphs.
However the graphs $R_1$ and $R_2$ are structurally different. 
For e.g., $R_1$ has a minimum vertex cover size of 5, while for $R_2$ it is 6.
As such, if any one of the graphs ($R_1$, say) is provided as input to be represented, then from the representation it is impossible to exaclty recover $R_1$'s structure. 
\end{proof}

\subsection{Proof of Proposition~\ref{ex: gcnn belong to model}} \label{sec: gcnn belong to model proof}
\begin{proof}
\citet{kipf2016semi} use a two layer graph convolutional network, in which each layer uses convolutional filters that aggregate information from the immediate neighborhood of the vertices.
This corresponds to a 2-local representation function $r(\cdot)$ in our computational model.
Following this step, the values at the vertices are aggregated using softmax to compute a probability score at each vertex.
Since this procedure is independent of the structure of the input graph, it is a valid gathering function $g(\cdot)$ in \model~and the overall architecture belongs to a 2-\model~model.

Similarly,~\citet{dai2017learning} also consider convolutional layers in which the neurons have a spatial locality of one.
Four such convolutional layers are cascaded together, the outputs of which are then processed by a separate $Q$-learning network.
Such a neural architecture is an instance of the 4-\model~model.
\end{proof}

\begin{figure}[t]
  \centerline{
     \subfigure[]{\includegraphics[height=30mm]{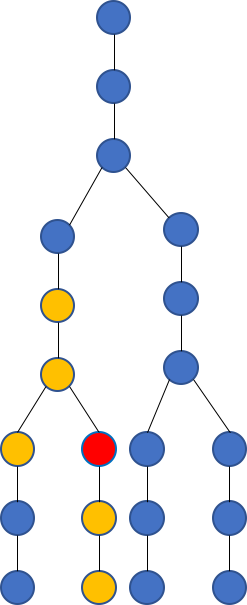}
    \label{fig: tree1}}
    \hfil
     \subfigure[]{\includegraphics[height=40mm]{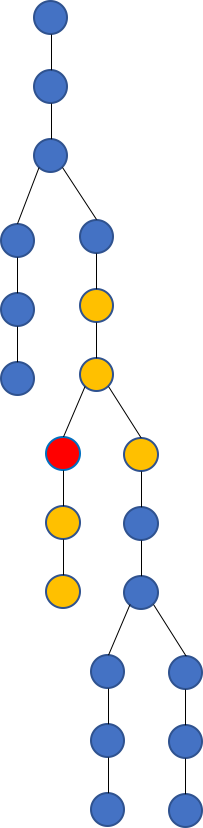}
    \label{fig: tree2}}}
    \caption{Example to illustrate $k$-\model~algorithms are insufficient for computing certain functions. Corresponding vertices in the two trees above have similar local neighborhoods, but the trees have minimum vertex cover of different sizes. }
    \label{fig: arch and tree egs}
\end{figure}

\subsection{Proof of Theorem~\ref{thm: fixed length insuf}} \label{sec: fixed length insuf proof}
\begin{proof}
Consider a family $\mathcal{G}$ of undirected, unweights graphs.
Let $f:\mathcal{G}\rightarrow \mathbb{Z}$ denote a function that computes the size of the minimum vertex cover of graphs from $\mathcal{G}$.
For $k>0$ fixed, let $\mathtt{ALG}$ denote any algorithm from the $k$-\model~model, with a representation function $r_\mathtt{ALG}(\cdot)$ and aggregating function $g_\mathtt{ALG}(\cdot)$.\footnote{See beginning of Section~\ref{sec: rep graphs as ts} for explanations of $r(\cdot)$ and $g(\cdot)$.}
We present two graphs $G_1$ and $G_2$ such that $f(G_1) \neq f(G_2)$, but the set of computed states $\{r_\mathtt{ALG}(v): v \in G_i\}$ is the same for both the graphs ($i=1,2$).
Now, since the gather function $g_\mathtt{ALG}(\cdot)$ operates only on the set of computed states (by definition of our model), this implies $\mathtt{ALG}$ cannot distinguish between $f(G_1)$ and $f(G_2)$, thus proving our claim.

For simplicity, we fix $k=2$ (the example easily generalizes for larger $k$).
We consider the graphs $G_1$ and $G_2$ as shown in Fig.~\ref{fig: tree1} and~\ref{fig: tree2} respectively.
To construct these graphs, we first consider binary trees $B_1$ and $B_2$ each having 7 nodes.
$B_1$ is a completely balanced binary tree with a depth of 2, whereas $B_2$ is a completely imbalanced binary tree with a depth of 3.
Now, to get $G_1$ and $G_2$, we replace each node in $B_1$ and $B_2$ by a chain of 3 nodes (more generally, by an odd number of nodes larger than $k$).
At each location in $B_i$ ($i=1,2$), the head of the chain of nodes connects to the tail of the parent's chain of nodes, as shown in Fig.~\ref{fig: arch and tree egs}.

The sizes of the minimum vertex cover of $G_1$ and $G_2$ are 9 and 10 respectively.
However, there exists a one-to-one mapping between the vertices of $G_1$ and the vertices of $G_2$ such that the $2$-hop neighborhood around corresponding vertices in $G_1$ and $G_2$ are the same.
For example, in Fig.~\ref{fig: tree1} and~\ref{fig: tree2} the pair of nodes shaded in red have an identical 2-hop neighborhood (shaded in yellow).
As such, the representation function $r_\mathtt{ALG}(\cdot)$ -- which for any node depends only on its $k$-hop neighborhood -- will be the same for corresponding pairs of nodes in $G_1$ and $G_2$.

Finally, the precise mapping between pairs of nodes in $G_1$ and $G_2$ is obtained as follows.
First consider a simple mapping between pairs of nodes in $B_1$ and $B_2$ in which (i) the 4 leaf nodes in $B_1$ are mapped to the leaf nodes in $B_2$, (ii) the root of $B_1$ is mapped to the root of $B_2$ and (iii) the 2 interior nodes of $B_1$ are mapped to the interior nodes of $B_2$.
We generalize this mapping to $G_1$ and $G_2$ in two steps: (1) mapping chains of nodes in $G_1$ to chains of nodes in $G_2$, according to the $B_1-B_2$ map, and (2) within corresponding chains of nodes, we map nodes according to order (head-to-head, tail-to-tail, etc.).
\end{proof}

\subsection{Sequential Heuristic to Compute MVC on Trees} \label{sec: seq heur trees}

Consider any unweighted, undirected tree $T$.
We let the state at any node $v\in T$ be represented by a two-dimensional vector $[x_v, y_v]$.
For any $v\in T$, $x_v$ takes values over the set $\{ -\epsilon, +1 \}$ while $y_v$ is in $\{-1, 0, \epsilon \}$.
Here $\epsilon$ is a parameter that we choose to be less than one over the maximum degree of the graph.
Semantically $x_v$ stands for whether vertex $v$ is `active' ($x_v=+1$) or `inactive' ($x_v=-\epsilon$).
Similarly $y_v$ stands for whether $v$ has been selected to be part of the vertex cover ($y_v = +\epsilon$), has not been selected to be part of the cover ($y_v = -1$), or a decision has not yet been made ($y_v = 0$).
Initially $x_v = -\epsilon$ and $y_v = 0$ for all vertices.
The heuristic proceeds in rounds, wherein at each round any vertex $v$ updates its state $[x_v, y_v]$ based on the state of its neighbors as shown in Algorithm~\ref{algo:seq heuristic}.

\begin{algorithm}[t]
   \caption{Sequential heuristic to compute minimum vertex cover on a tree.}
   \label{algo:seq heuristic}
\begin{algorithmic}
   \STATE {\bfseries Input:} Undirected, unweighted tree $T$; Number of rounds {\tt NumRounds}
   \STATE {\bfseries Output:} Size of minimum vertex cover on $T$
   \STATE $x_v(0) \leftarrow -\epsilon$ for all $v\in T$ \COMMENT{$x_v(i)$ is $x_v$ at round $i$}
   \STATE $y_v(0) \leftarrow 0$ for all $v\in T$ \COMMENT{$y_v(i)$ is $y_v$ at round $i$}
   \STATE \slash\slash ~Computing the representation $r(v)$ for each $v\in T$
   \FOR{i from 1 {\bfseries to} {\tt NumRounds}}
   	\STATE At each vertex $v$:
   		\IF{$\sum_{u\in\Gamma(v)} x_u(i-1) \geq -\epsilon $}
   			\STATE $x_v(i) \leftarrow +1$
			\IF{$\sum_{u\in\Gamma(v)}y_u(i-1) < 0$}
				\STATE $y_v(i) \leftarrow \epsilon$
			\ELSE
				\STATE $y_v(i) \leftarrow -1$
			\ENDIF
   		\ELSE
   			\STATE $x_v(i) \leftarrow -\epsilon$ and $y_v(i) \leftarrow 0$
   		\ENDIF
   \ENDFOR
   \STATE \slash\slash ~Computing the aggregating function $g(\{r(v): v\in T\})$
   \STATE $\bar{y}_v \leftarrow \left( \sum_{i=1}^{\mathtt{NumRounds}} (y_v(i) + 1)/(1+\epsilon) \right)/\mathtt{NumRounds} $
   \STATE {\bf return} $\sum_{v \in T} \bar{y}_v $
\end{algorithmic}
\end{algorithm}


The update rules at vertex $v$ are (1) if $v$ is a leaf or if at least one of $v$'s neighbors are active, then $v$ becomes active; (2) if $v$ is active, and if at least one of $v$'s active neighbors have not been chosen in the cover, then $v$ is chosen to be in the cover; (3) if all of $v$'s neighbors are inactive, then $v$ remains inactive and no decision is made on $y_v$.

At the end of the local computation rounds, the final vertex cover size is computed by first averaging the $y_v$ time-series at each $v\in T$ (with translation, and scaling as shown in Algorithm~\ref{algo:seq heuristic}), and then summing over all vertices.

\section{Section~\ref{sec: Neural Network Design} Details}
\label{sec: net arch appendix}

\begin{algorithm}[t]
   \caption{Testing procedure of \algo~on a graph instance.}
   \label{algo: graph2seq testing}
\begin{algorithmic}
   \STATE {\bfseries Input:} graph $G$, trained parameters, objective $f:G\rightarrow \mathbb{R}$ we seek to maximize, maximum sequence length $T_\mathrm{max}$
   \STATE {\bfseries Output:} solution set $S\subseteq V$
   \STATE $S_\mathrm{opt} \leftarrow \{ \}$, $v_\mathrm{opt} \leftarrow 0$ \COMMENT{initialize}
   \FOR{$T$ from $1$ {\bfseries to} $T_\mathrm{max}$}
   	\STATE $S \leftarrow $ solution returned by \algo($T$)
   		\IF{$f(S) > v_\mathrm{opt}$}
   			\STATE $S_\mathrm{opt} \leftarrow S$
			\STATE $v_\mathrm{opt} \leftarrow f(S)$
   		\ENDIF
   \ENDFOR
   \STATE {\bf return} $S_\mathrm{opt}$
\end{algorithmic}
\end{algorithm}


\subsection{Reinforcement Learning Formulation} \label{sec: learning algorithm}
Let $G(V,E)$ be an input graph instance for the optimization problems mentioned above. 
Note that the solution to each of these problems can be represented by a set $S \subseteq V$. 
In the case of the minimum vertex cover (MVC) and maximum independent set (MIS), the set denotes the desired optimal cover and independent set respectively; 
for max cut (MC) we let $(S, S^c)$ denote the optimal cut. 
For the following let  $f:2^V \rightarrow \mathbb{R}$ be the objective function of the problem (i.e., MVC, MC or MIS) that we want to maximize, and let $\mathcal{F} \subseteq 2^V$ be the set of feasible solutions. 

{\bf Dynamic programming formulation.} Now, consider a dynamic programming heuristic in which the subproblems are defined by the pair $(G, S)$, where $G$ is the graph and $S\subseteq V$ is a subset of vertices that have already been included in the solution.  
For a vertex $a\in V \backslash S$ let $ Q(G,S, a) = \max_{S' \supseteq S\cup\{a\}, S' \in \mathcal{F}}f(S') - f(S\cup \{a\}) $ denote the marginal utility gained by selecting vertex $a$. 
Such a $Q$-{\em function} satisfies the Bellman equations given by
\begin{align}
Q(G,S,a) = &f(S\cup\{a\}) - f(S) \notag \\
&+ \max_{a' \in V \backslash S\cup\{a\} }Q(G, S\cup\{a\}, a'). \label{eq: Bellman equations}
\end{align}
It is easily seen that computing the $Q$-functions solves the optimization problem, as $\max_{S \in \mathcal{F}} f(S) = \max_{a\in V} Q(G, \{\}, a)$.   
However exactly computing $Q$-functions may be computationally expensive. 
One approach towards approximately computing $Q(G,S,a)$ is to fit it to a (polynomial time computable) parametrized function, in a way that an appropriately defined error metric is minimized. 
This approach is called $Q$-{\em learning} in the reinforcement learning (RL) paradigm, and is described below. 

{\bf State, action and reward.} We consider a reinforcement learning policy in which the solution set $S\subseteq V$ is generated one vertex at a time. 
The algorithm proceeds in rounds, where at round $t$ the RL agent is presented with the graph $G$ and the set of vertices $S_t$ chosen so far. 
Based on this {\em state} information, the RL agent outputs an {\em action} $A_t \in  V \backslash S_t$. 
The set of selected vertices is updated as $S_{t+1} = S_t \cup \{A_t\}$. 
Initially $S_0 = \{ \}$. 
Every time the RL agent performs an action $A_t$ it also incurs a {\em reward} $R_t = f(S_t\cup \{A_t\}) - f(S_t)$.
Note that the $Q$-function $Q(G, S_t, a)$ is well-defined only if $S_t$ and $a$ are such that there exists an $S' \supseteq S_t \cup \{a\}$ and $S' \in \mathcal{F}$. 
To enforce this let $\mathcal{F}_t = \{ a \in V \backslash S_t: \exists~ S' \text{ s.t. } S' \supseteq S_t \cup \{a\} \text{ and } S' \in \mathcal{F} \} $ denote the set of  feasible actions at time $t$.
Each round, the learning agent chooses an action $A_t \in \mathcal{F}_t$. 
The algorithm terminates when $\mathcal{F}_t = \{ \}$.  

{\bf Policy.} 
The  goal of the RL agent is to learn a {\em policy} for selecting actions at each time, such that the cumulative reward incurred $\sum_{t\geq 0} R_t$ is maximized. 
A measure of the generalization capability of the policy is how well it is able to maximize cumulative reward for different graph instances from a collection (or from a  distribution) of interest. 
 
{\bf $Q$-learning.}
Let $\tilde{Q}(G, S, a; \Theta)$ denote the approximation of $Q(G,S,a)$ obtained using a parametrized function with parameters $\Theta$. 
Further let $\left( (G_i, S_i, a_i) \right)_{i=1}^N$ denote a sequence of (state, action) tuples available as training examples. 
We define empirical loss as 
\begin{align}
\hat{L} = &\sum_{i=1}^N\left( \tilde{Q}(G_i,S_i,a_i; \Theta) - f(S_i\cup\{a_i\}) + f(S_i) \right. \notag \\
&\left. - \max_{a' \in V \backslash S_i\cup\{a_i\} } \tilde{Q}(G_i, S_i\cup\{a_i\}, a'; \Theta) \right)^2, \label{eq: emp loss}
\end{align}
and minimize using stochastic gradient descent. 
The solution of the Bellman equations~\eqref{eq: Bellman equations} is a stationary point for this optimization. 

{\bf Remark.} 
Heuristics such as ours, which select vertices one at a time in an irreversible fashion are studied as `priority greedy' algorithms in computer science literature~\citep{borodin2003incremental, angelopoulos2003randomized}.
The fundamental limits (worst-case) of priority greedy algorithms for minimum vertex cover and maximum independent set has been discussed in~\citet{borodin2010priority}. 

\subsection{Seq2Vec Update Equations}

{\bf Seq2Vec.}
The sequence $\left( \{\mathbf{x}_v(t): v \in V\} \right)^T_{t=1}$ is processed by a  gated  recurrent network that sequentially reads $\mathbf{x}_v(\cdot)$ vectors at each time index for all $v\in V$.
Standard GRU~\citep{chung2014empirical}.
For time-step $t\in\{1,\ldots,T\}$, let $\mathbf{y}_v(t) \in \mathbb{R}^d$ be the $d$-dimensional cell state, $\mathbf{i}_v(t)\in\mathbb{R}^d$ be the cell input and $\mathbf{f}_v(t) \in (0,1)^d$  be the forgetting gate, for each vertex $v\in V$.
Each time-step a fresh input $\mathbf{i}_v(t)$ is computed based on the current states $\mathbf{x}_u(t)$ of $v$'s neighbors in $G$.
The cell state is updated as a convex combination of the freshly computed inputs $\mathbf{i}_v(t)$ and the previous cell state $\mathbf{y}_v(t-1)$, where the weighting is done according to a forgetting value $\mathbf{f}_v(t)$ that is also computed based on the current vertex states.
The update equations for the input vector, forgetting value and cell state are chosen as follows:
\begin{align}
\mathbf{i}_v(t+1) &= \mathrm{relu} ( \mathbf{W}_4 \sum_{u\in\Gamma(v)} \mathbf{x}_v(t) + \mathbf{w}_5 c_v(t) + \mathbf{b}_6 ) \notag \\
\mathbf{f}_v(t+1) &= \mathrm{sigmoid} ( \mathbf{W}_7 \sum_{u\in V} \mathbf{x}_u(t) + \mathbf{b}_8 ) \notag \\
\mathbf{y}_v(t+1) &= \mathbf{f}_v(t+1) \odot \mathbf{i}_v(t+1) + (\mathbf{1}-\mathbf{f}_v(t+1)) \odot \mathbf{y}_v(t),
\end{align}
where $\mathbf{W}_4, \mathbf{W}_7 \in \mathbb{R}^{d\times d}$ and $\mathbf{w}_5, \mathbf{b}_6, \mathbf{b}_8 \in \mathbb{R}^d$ are trainable parameters,  $t=0,1,\ldots,T-1$, and $\mathbf{1}$ denotes the $d$-dimensional all-ones vector, and $\odot$ is element-wise multiplication. $\mathbf{y}_v(0)$ is initialized to all-zeros for every $v\in V$.
The cell state at the final time-step $\mathbf{y}_v(T), v\in V$ is the desired vector summary of the \algo~sequence.

\section{Evaluation Details} \label{apx: eval}

\subsection{Heuristics compared}
\label{sec: heuristics}
We compare \gtos~against: 

(1) {\em Structure2Vec}~\citep{dai2017learning}, a spatial GCNN with depth of 5. \\
(2) {\em GraphSAGE}~\citep{hamilton2017inductive} using (a) GCN, (b) mean and (c) pool aggregators, with the depth restricted to 2 in each case. \\ 
(3) {\em WL kernel NN}~\citep{lei2017deriving}, a neural architecture that embeds the WL graph kernel, with a depth of 3 and width of 4 (see~\citet{lei2017deriving} for details). \\
(4) {\em WL kernel embedding}, in which the feature map corresponding to WL subtree kernel of the subgraph in a 5-hop neighborhood around each vertex is used as its vertex embedding~\citep{shervashidze2011weisfeiler}.
Since we test on large graphs, instead of using a learned label lookup dictionary we use a standard SHA hash function for label shortening at each step. 
In each of the above models, the outputs of the last layer are fed to a $Q$-learning network, and trained the same way as \gtos. \\
(5) {\em Greedy algorithms.}
We consider greedy heuristics~\citep{williamson2011design} for each of MVC, MC and MIS. \\
(6) {\em List heuristic.} A fast list-based algorithm proposed recently in~\citet{shimizu2016fast} for MVC and MIS. \\
(7) {\em Matching heuristic.} A $2$-approximation algorithm for MVC~\citep{williamson2011design}.


\subsection{Adversarial Training} \label{sec: adv training} 
\begin{figure*}[t]
  \centerline{
     \subfigure[]{\includegraphics[height=25mm]{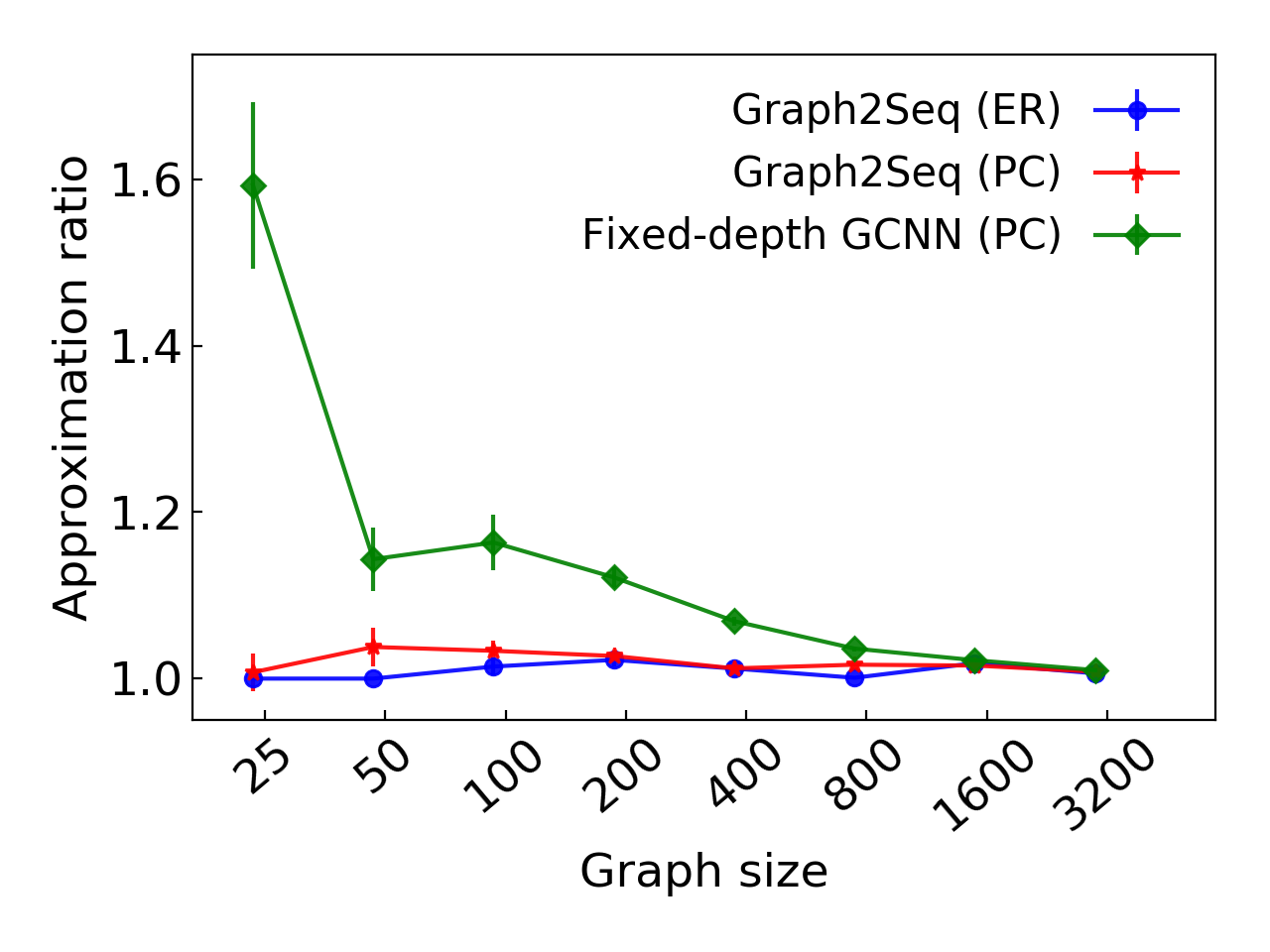}
    \label{fig: adv erdos}}  
     \subfigure[]{\includegraphics[height=25mm]{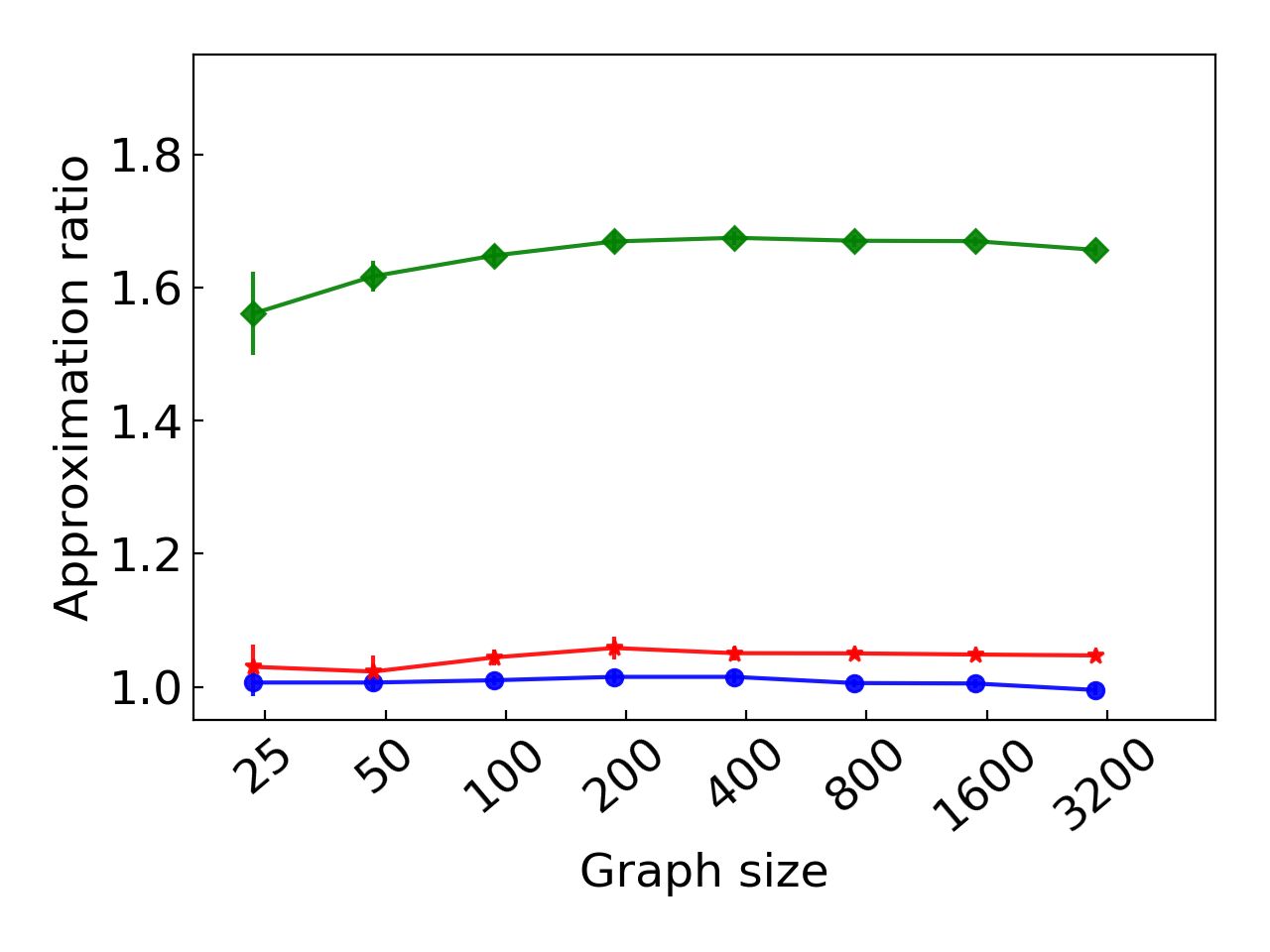}
    \label{fig: adv regular}}
     \subfigure[]{\includegraphics[height=25mm]{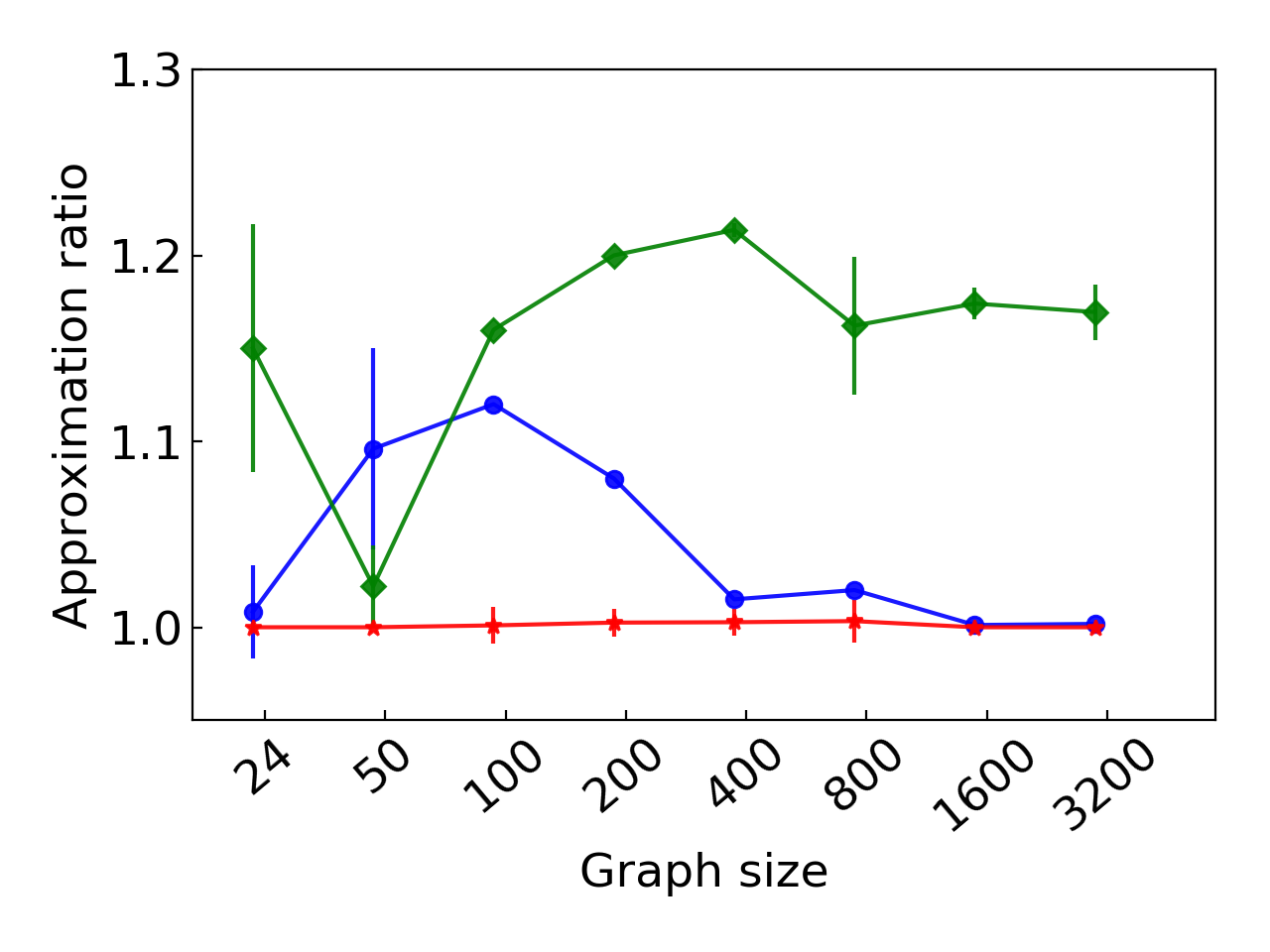}
     \label{fig: adv bipartite}}
     \subfigure[]{\includegraphics[height=25mm]{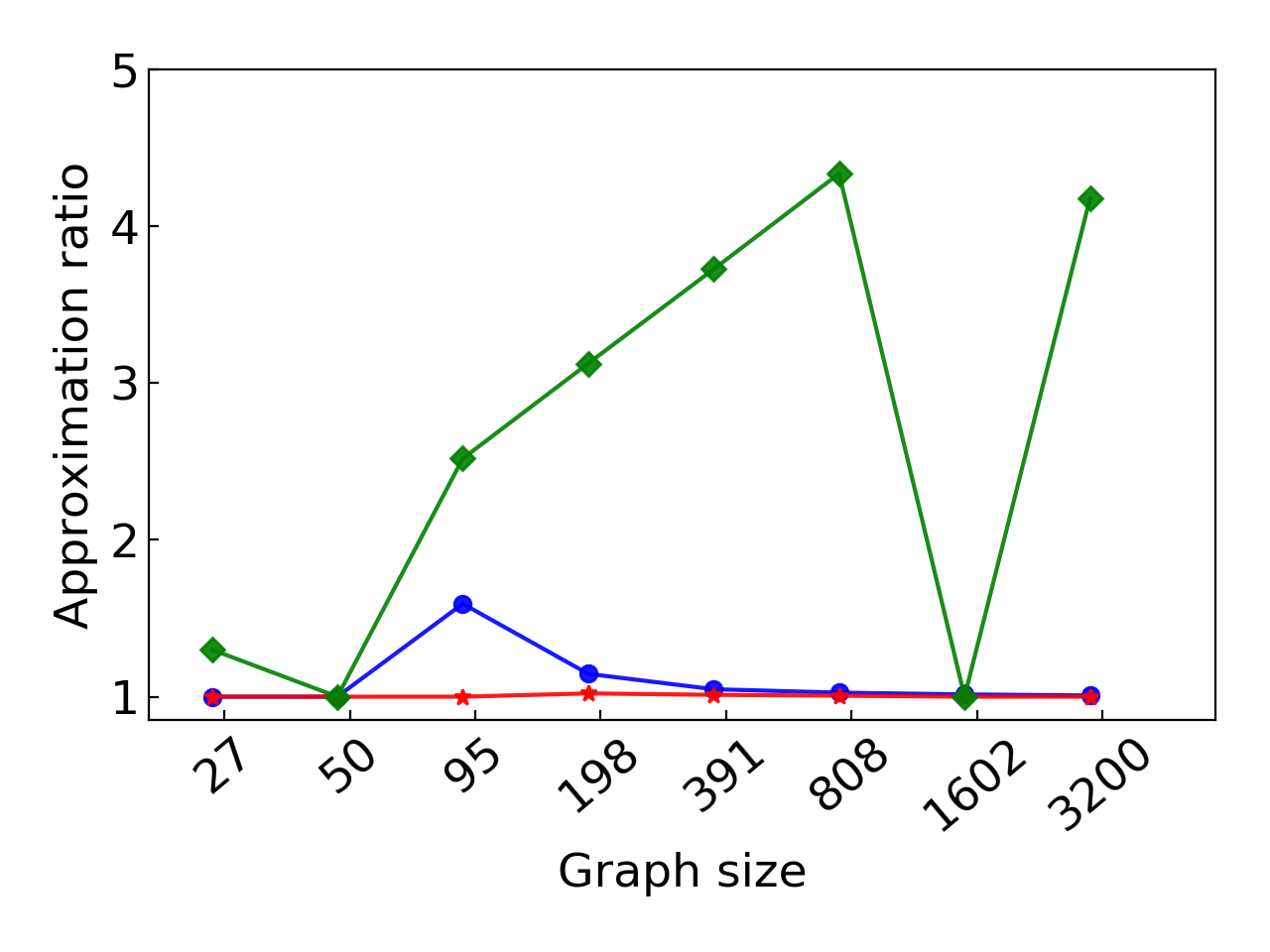}
     \label{fig: adv greedy}}}
    \caption{\small Minimum vertex cover in (a) random Erdos-Renyi graphs, (b) random regular graphs, (c) random bipartite graphs, (d) greedy example, under Erdos-Renyi graph and adversarial graph training strategies.}
    \label{fig: adv training}
\end{figure*}

So far we have seen the generalization capabilities of a \gtos~model trained on small Erdos-Renyi graphs.
In this section we ask the question: is even better generalization possible by training on a different graph family? 
The answer is in the affirmative. 
We show that by training on {\em planted vertex-cover} graph examples---a class of `hard' instances for MVC---we can realize further generalizations. 
A planted-cover example is a graph, in which a small graph is embedded (`planted') within a larger graph such that the vertices of the planted graph constitute the optimal minimum vertex cover.
Figure~\ref{fig: adv training} shows the result of testing \gtos~models trained under both Erdos-Renyi and planted vertex cover graphs. 
While both models show good scalability in Erdos-Renyi and regular graphs, on bipartite graphs and worst-case graphs the model trained on planted-cover graphs shows even stronger consistency by staying 1\% within optimal.

\subsection{Geometry of Encoding and Semantics of \algo} \label{sec: vc geometry and sem}

\begin{figure*}[th]
  \centerline{\subfigure[]{\includegraphics[height=35mm]{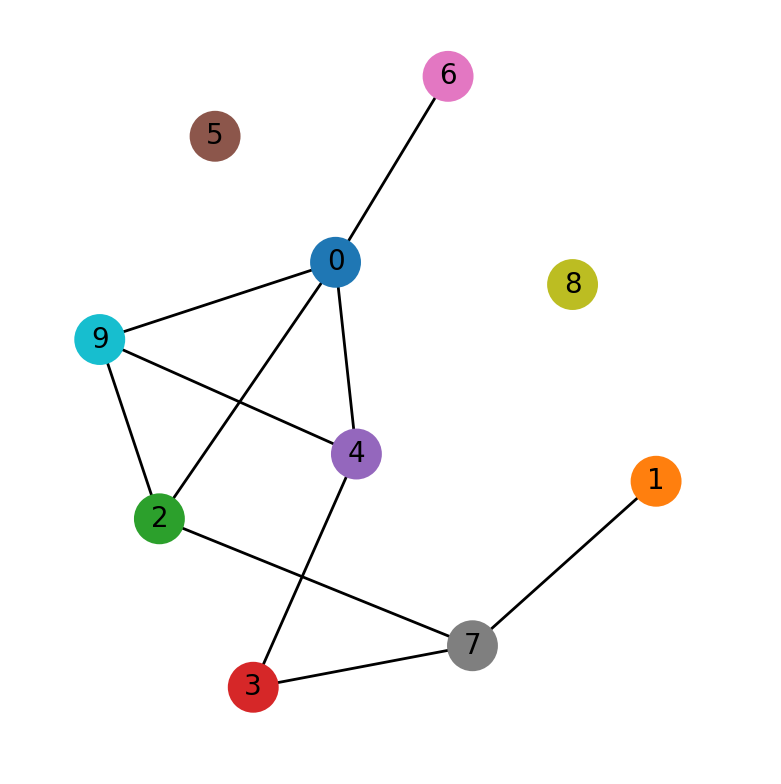}
     \label{fig: er_graph}}
     \subfigure[]{\includegraphics[height=40mm]{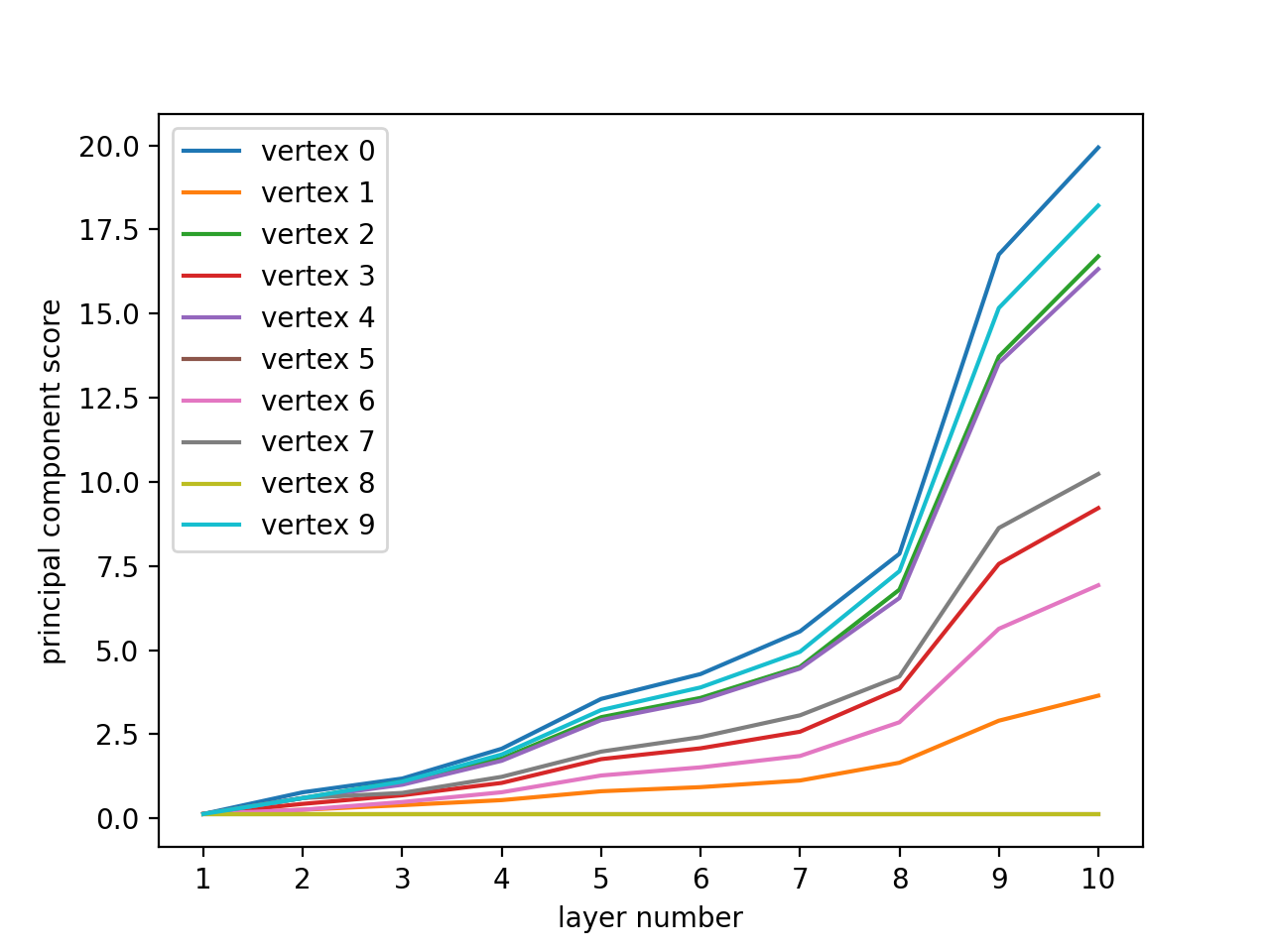}
     \label{fig: er_pca}}
     \subfigure[]{\includegraphics[height=40mm]{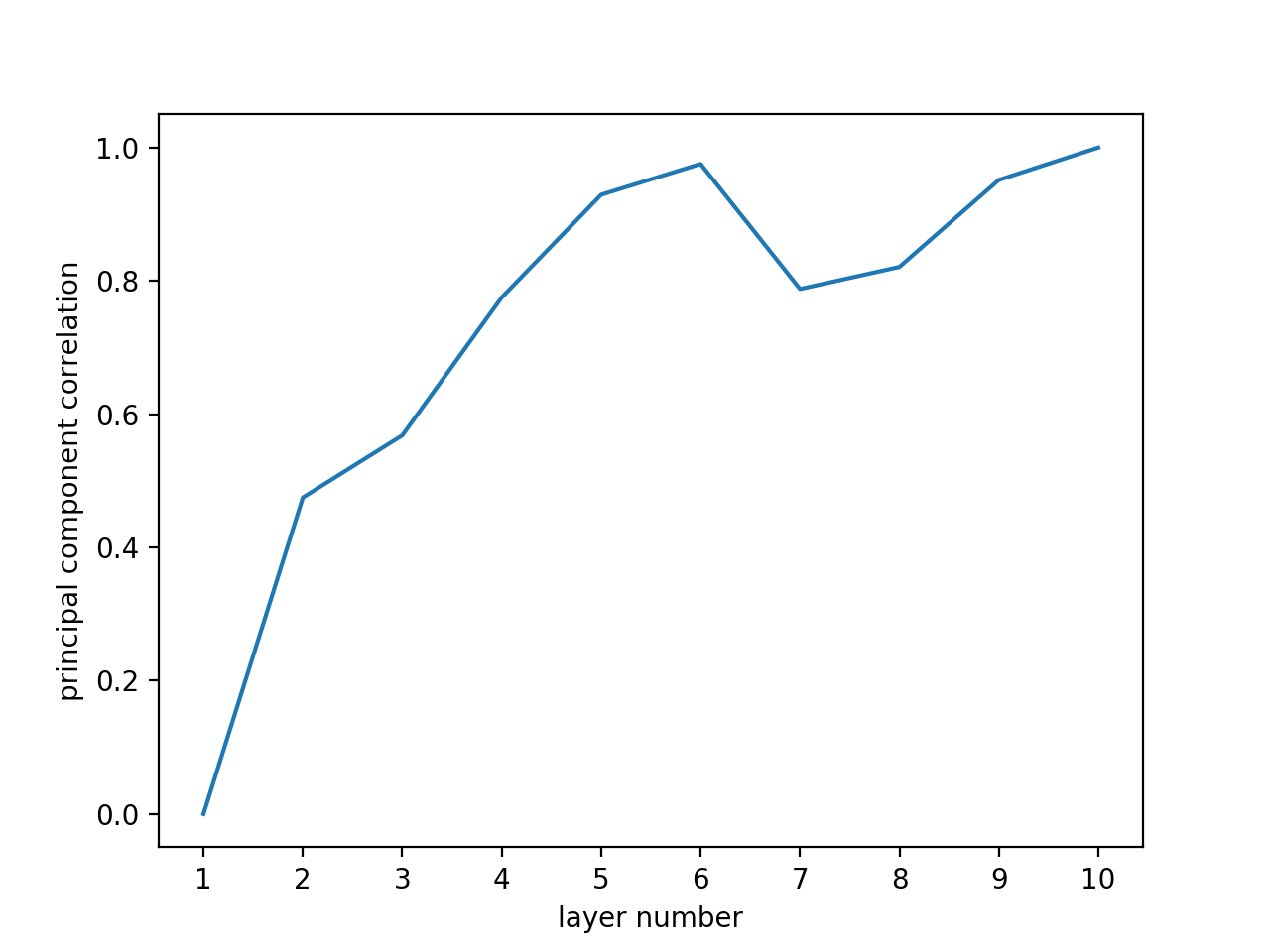}
    \label{fig: er_pca_inn_prod}}}
    \caption{(a) Erdos-Renyi graph of size 10 considered in Figures (b) and (c), (b) vertex-wise principal component scores at each layer, and (c) projection of the principal direction at each iteration on the principal direction of iteration 10. These experiments are performed on our trained model.}
    \label{fig: pca_trained}
\end{figure*}

Towards an understanding of what aspect of solving the MVC is learnt by \algo, we conduct empirical studies on the dynamics of the state vectors as well as  present techniques and semantic interpretations of \algo. 

In the first set of experiments, we investigate the vertex state vector sequence.
We consider graphs of size up to 50 and of types discussed in Section~\ref{sec: evals}.
For each fixed graph, we observe the vertex state $\mathbf{x}(\cdot)$ (Equation~\ref{eq: basic filter}) evolution to a depth of 10 layers.

{\bf (1) Dimension collapse.}
As in the random parameter case, we observe that on an average more than 8 of the 16 dimensions of the vertex state become zeroed out after 4 or 5 layers.

{\bf (2) Principal components' alignment.}
The principal component direction of the vertex state vectors at each layer converges.
Fig.~\ref{fig: er_pca_inn_prod} shows this effect for the graph shown in Fig.~\ref{fig: er_graph}.
We plot the absolute value of the inner product between the principal component direction at each layer and the principal component direction at layer 10.

{\bf (3) Principal component scores and local connectivity.}
The component of the vertex state vectors along the principal direction roughly correlate to how well the vertex is connected to the rest of the graph.
We demonstrate this again for the graph shown in Fig.~\ref{fig: er_graph}, in Fig~\ref{fig: er_pca}.


{\bf (4) Optimal depth}.
 We study the effect of depth on approximation quality on the  four graph types being tested (with size 50); we plot the vertex cover quality as returned by \algo as we vary the number of layers up to 25.
Fig.~\ref{fig: layerwise} plots the results of this experiment, where there is no convergence behavior but nevertheless  apparent that different graphs work optimally at different layer values.
While the optimal layer value is 4 or 5 for  random bipartite and random regular graphs,  the worst case greedy example requires 15 rounds.
This experiment underscores the importance of having a {\em flexible} number of layers is better than a fixed number; this is only enabled by the time-series nature of \algo~and is inherently missed by the fixed-depth GCNN representations in the literature.


{\bf (5) $Q$-function semantics.}
Recall that the $Q$-function of~\eqref{eq: Q func est} comprises of two terms.
The first term, denoted by $Q_1$,  is the same for all the vertices and includes a sum of all the $\mathbf{y}(\cdot)$ vectors. The second term, denoted by $Q_2(v)$ depends on the $\mathbf{y}(\cdot)$ vector for the vertex being considered. In this experiment we plot these two values at the very first layer of the learning algorithm (on a planted vertex cover graph of size 15, same type as in the training set) and make the following observations: (a) the values of $Q_1$ and $Q_2(\cdot)$ are close to being integers.
$Q_1$ has a value that is one less than the negative of the minimum vertex cover size.
(b) For a vertex $v$, $Q_2(v)$ is binary valued from the set $\{0, 1\}$.
$Q_2(v)$ is one, if vertex $v$ is part of an optimum vertex cover, and zero otherwise.
Thus the neural network, in principle, computes the complete set of vertices in the optimum cover at the very first round itself.

\begin{figure*}[t]
  \centerline{\subfigure[]{\includegraphics[height=46mm]{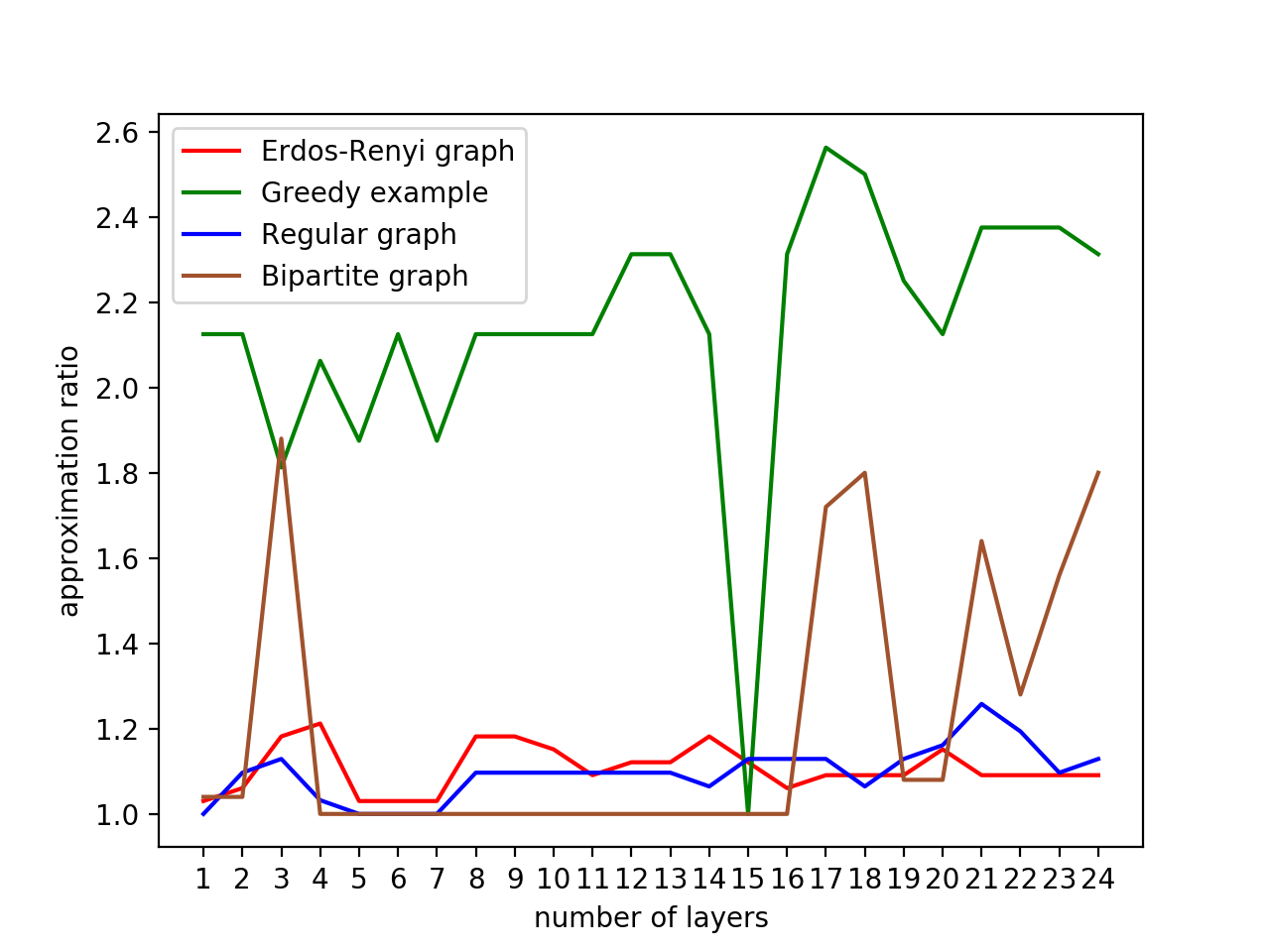}
     \label{fig: layerwise}}
     \subfigure[]{\includegraphics[height=40mm]{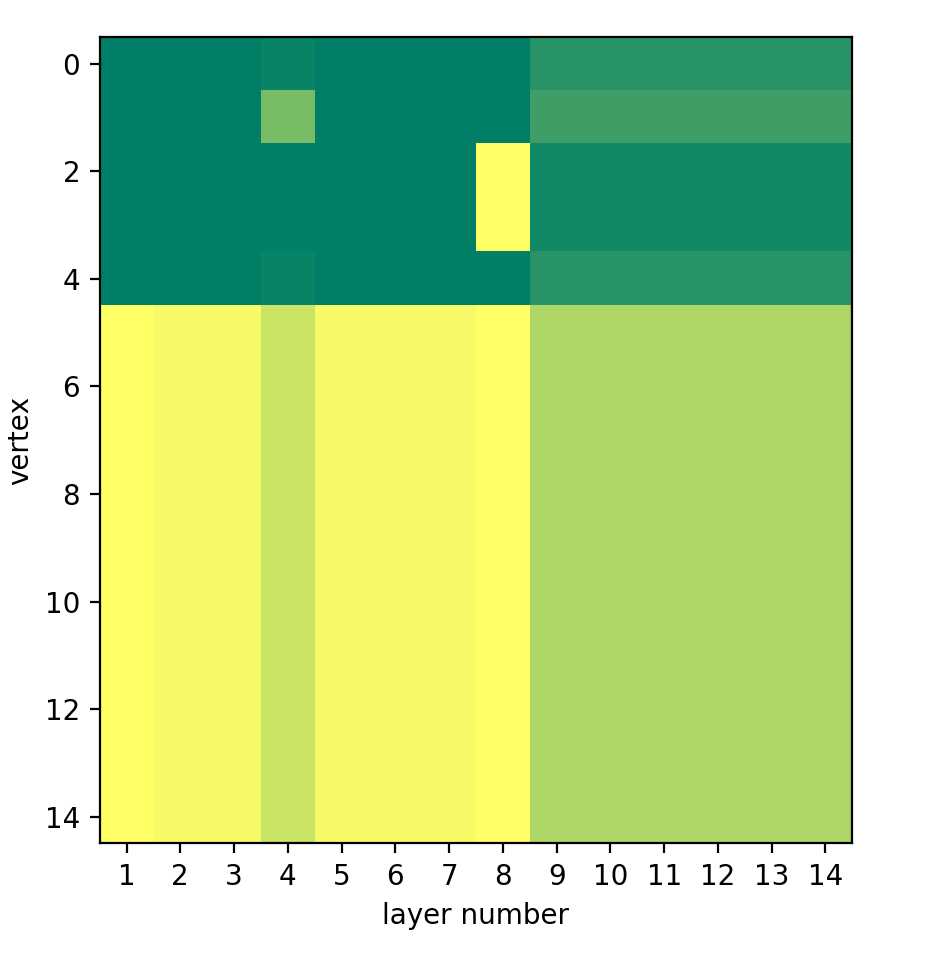}
     \label{fig: q_val_ours_pc}}
     \subfigure[]{\includegraphics[height=40mm]{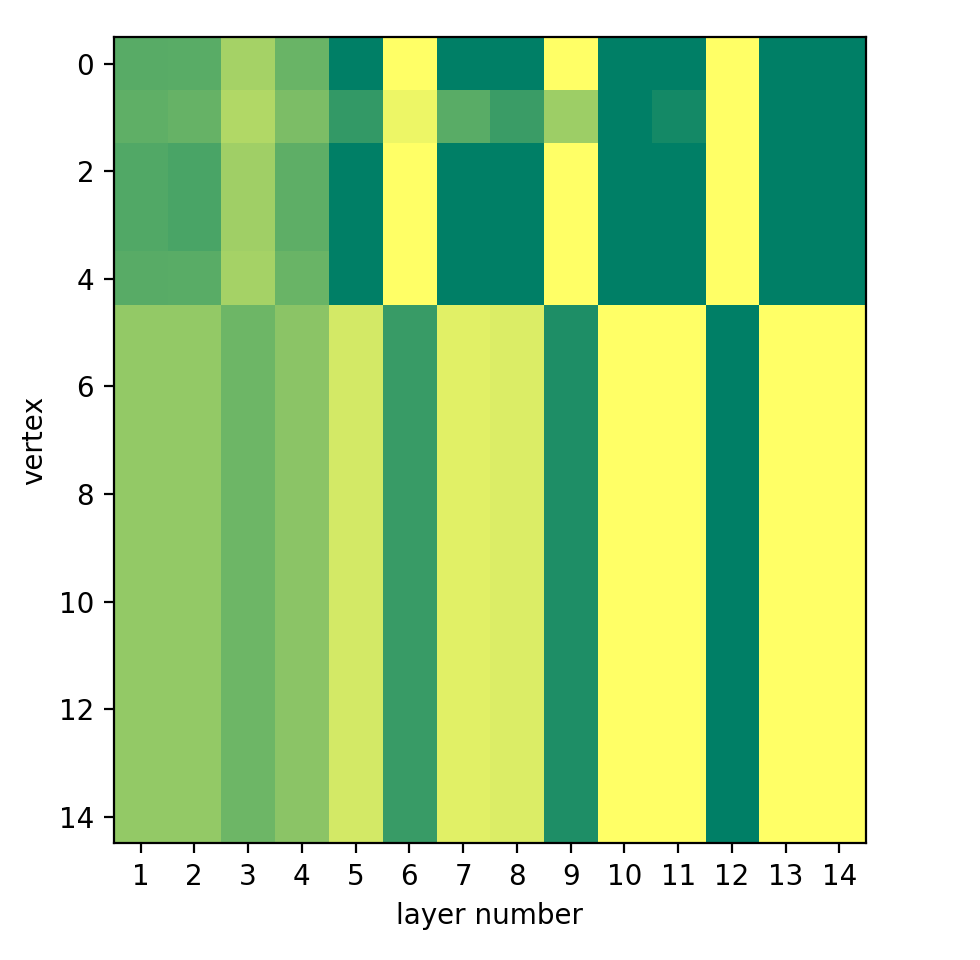}
    \label{fig: q_val_lsong_pc}}}
    \caption{(a) Approximation ratio of \algo~with varying number of layers, (b) $\mathbf{y}(\cdot)$ vectors of \algo~in the intermediate layers seen using the $Q$-function, (c) $\mathbf{x}(\cdot)$ vectors of the fixed-depth model seen using the $Q$-function. Figure (b) and (c) are on planted vertex cover graph with optimum cover of vertices $\{0,1,2,3,4\}$. }
    \label{fig: layerwise q_val}
\end{figure*}

{\bf (6) Visualizing the learning dynamics}.  The above observations suggests to `visualize' how our learning algorithm proceeds in each layer of the evolution using the lens of the value of
  $Q_2(\cdot)$. 
In this  experiment, we consider size-15 planted vertex cover graphs on (i) \algo, and (ii) the fixed-depth GCNN  trained on planted vertex cover graphs.
Fig.~\ref{fig: q_val_ours_pc} and~\ref{fig: q_val_lsong_pc} show the results of this experiment.
The planted vertex cover graph considered for these figures has an optimal vertex cover comprising vertices $\{0,1,2,3,4\}$.
We center (subtract mean) the $Q_2(\cdot)$ values at each layer, and threshold them to create the visualization.
A dark green color signifies the vertex has a high $Q_2(\cdot)$ value, while the yellow means a low $Q_2(\cdot)$ value.
We can see that in \algo~the heuristic is able to compute the optimal cover, and moreover this answer does not change with more rounds.
 The fixed depth GCNN has a non-convergent answer which oscillates between  a complementary set of vertices. Take away message:  having an upper LSTM layer in the learning network is critical to identify when an optimal solution is reached in the evolution, and ``latch on" to it.


\end{document}